\documentclass[letterpaper]{article} 
\usepackage{aaai23}  
\usepackage{times}  
\usepackage{helvet}  
\usepackage{courier}  
\usepackage[hyphens]{url}  
\usepackage{graphicx} 
\usepackage{natbib}  
\usepackage{caption} 
\usepackage{algorithm}
\usepackage{algorithmic}

\usepackage{subcaption}
\usepackage{enumitem}
\usepackage{mathtools}
\usepackage{amssymb, amsmath, amsthm}
\usepackage{tikz}
\usepackage{makecell}
\usepackage{diagbox}
\usepackage{multirow}
\usepackage{multicol}
\usepackage{yang-macros}

\usepackage{newfloat}
\usepackage{listings}
\DeclareCaptionStyle{ruled}{labelfont=normalfont,labelsep=colon,strut=off} 
\floatstyle{ruled}
\newfloat{listing}{tb}{lst}{}
\floatname{listing}{Listing}
\title{Minimax  AUC Fairness: Efficient Algorithm With Provable Convergence}
\author{
    Zhenhuan Yang\textsuperscript{\rm 1}\thanks{This work was mostly done while Zhenhuan Yang was a student at University at Albany, SUNY.}
    Yan Lok Ko\textsuperscript{\rm 2}
    Kush R. Varshney\textsuperscript{\rm 3}
    Yiming Ying\textsuperscript{\rm 2}\thanks{Corresponding Author.}
}
\affiliations{
    \textsuperscript{\rm 1}Etsy, Inc, Brooklyn, New York, USA\\

    \textsuperscript{\rm 2}University at Albany, SUNY, Albany, New York, USA\\
    \textsuperscript{\rm 3}IBM Research, Yorktown Heights, New York, USA\\
    zhenhuan.yang@hotmail.com, yko@albany.edu, krvarshn@us.ibm.com, yying@albany.edu
}

\begin{document}

\maketitle

\begin{abstract}

The use of machine learning models in consequential decision making often exacerbates societal inequity, in particular yielding disparate impact on members of marginalized groups defined by race and gender. The area under the ROC curve (AUC) is widely used to evaluate the performance of a scoring function in machine learning, but is studied in algorithmic fairness less than other performance metrics. Due to the pairwise nature of the AUC, defining an AUC-based group fairness metric is pairwise-dependent and may involve both \emph{intra-group} and \emph{inter-group} AUCs. Importantly, considering only one category of AUCs is not sufficient to mitigate unfairness in AUC optimization. In this paper, we propose a minimax learning and bias mitigation framework that incorporates both intra-group and inter-group AUCs while maintaining utility. Based on this Rawlsian framework, we design an efficient stochastic optimization algorithm and prove its convergence to the minimum group-level AUC. We conduct numerical experiments on both synthetic and real-world datasets to validate the effectiveness of the minimax framework and the proposed optimization algorithm.

\end{abstract}

\section{Introduction}
Recent years have witnessed an increasing recognition that allocation decisions unfavorable to people from vulnerable groups (defined by sensitive attributes such as race, gender, and age) worsen with the use of machine learning. Alongside, a burgeoning set of mitigation algorithms has been developed  \citep{agarwal2018reductions,calders2009building,calmon2017optimized,chouldechova2020snapshot,donini2018empirical,hardt2016equality,lohia2019bias,pleiss2017fairness,zafar2017fairness}. Recognizing that they are only a small sliver of all actions that may be taken when viewing the program of justice holistically, the aim of bias mitigation is to ensure that the output of a classifier is not dependent on sensitive attributes. Most existing work focuses on statistical fairness metrics composed of entries of the classifier's confusion matrix. 

On another front of machine learning, the area under the ROC curve (AUC) \citep{hanley1982meaning} is one of the most widely used performance metrics in classification tasks with class-imbalance and when the relative costs of false positives and false negatives are difficult to pin down, and in bipartite ranking tasks. Learning a scoring function by maximizing its AUC --- instead of the accuracy --- has received increasing  attention  \citep{Cortes2003AUCOV,gao2013one,ying2016stochastic, liu2019stochastic,lei2021stochastic,yang2022auc,Zhaoicml11,yang2021learning}.  
However, there is not yet much work on AUC-related fairness in machine learning. 

Properly defining group-level AUC fairness leads to two categories of metrics, depending on the specific groups to which the positive and negative examples belong. The first type, \textit{intra-group} AUC, constrains both positive and negative examples to the same group. The second type, \textit{inter-group} AUC,\footnote{\citet{kallus2019fairness} originally name this xAUC.} computes the metric with positive and negative examples being from different groups. On one hand, by only focusing on intra-group AUC, one does not fully account for all possible disparate impacts.\footnote{We use the term disparate impact in the general sense of inequality in outcomes across groups, not in the specific sense of the ratio of selection rates across groups.} Indeed, as observed by  \citet{kallus2019fairness}, similar intra-group AUCs may still lead to  a disparate impact where  positive examples of one group are misranked below negative examples of another group. On the other hand,  we also witness from the synthetic experiment in Figure \ref{fig:limitation-inter} that solely relying on inter-group AUC fairness can overlook unfairness with respect to the intra-group AUC. 
\begin{figure*}[!ht]
\begin{subfigure}{0.49\linewidth}
    \includegraphics[width=\linewidth]{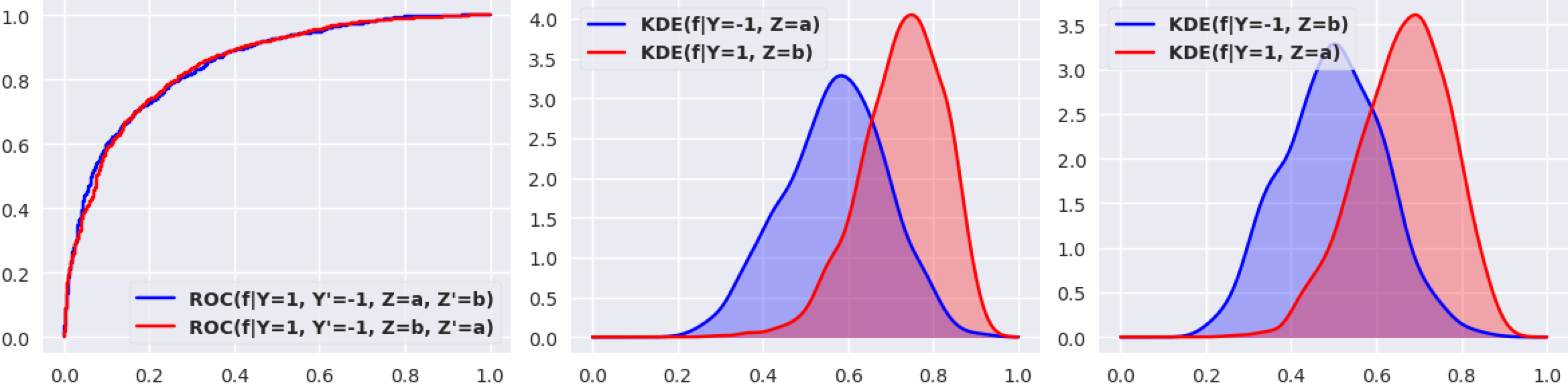}
    \caption{\it Inter-group ROC curves and Positive/negative KDE comparison}
\end{subfigure}%
    \hfill%
\begin{subfigure}{0.49\linewidth}
    \includegraphics[width=\linewidth]{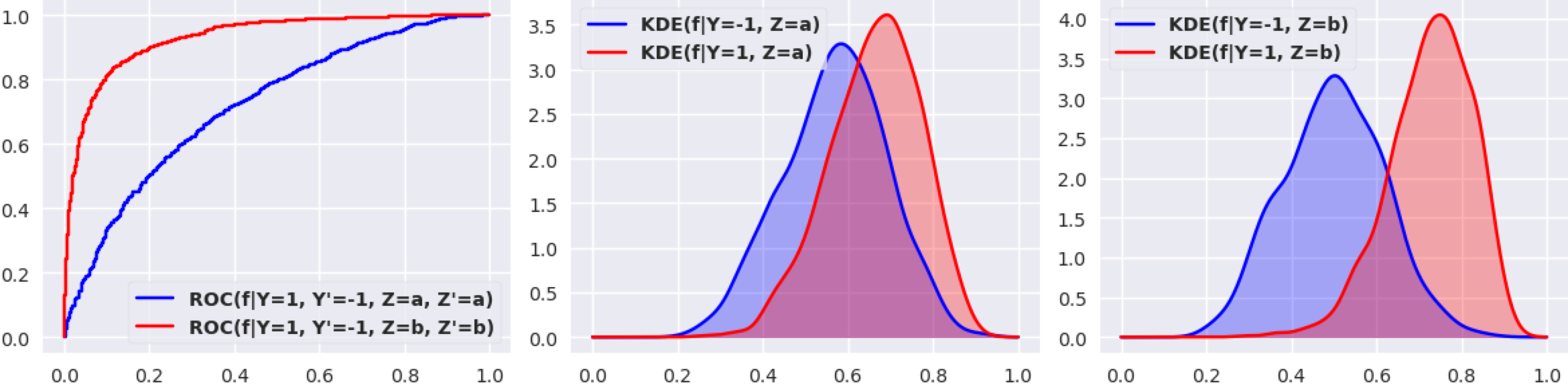}
    \caption{\it Intra-group ROC curves and Positive/negative KDE comparison}
\end{subfigure}
\caption{\it Illustration of the discrepancy between inter-group AUC and intra-group AUC. In this example, inter-group AUC is in a fair situation and intra-group AUC is in an unfair situation. Here $f$ denotes the synthetic scoring function sampled from Gaussian distribution and KDE denote the kernel density estimation of $f$, $Y$ denotes the class label and $Z$ denotes the protected attribute (see Appendix \ref{sec:exp} for more details). In part (a), the gap of the peaks in KDE is obvious. In part (b), the gap of the peaks is small and the overlap is large when $Z=Z'=a$, which indicates probability of positive sample being misranked than negative sample is higher than $Z=Z'=b$. }
\label{fig:limitation-inter}
\vspace{-10pt}
\end{figure*}
These two observations strongly suggest that to mitigate disparate impact when the performance metric is the  AUC score,  one should consider \emph{both} inter-  and intra-group AUC fairness during the learning process.

In this paper, we follow the Rawlsian principle of maximin welfare for distributive justice \citep{rawls2001justice} and  formalize our fairness goal as follows: 
\begin{center}
\textit{Find a scoring function that maximizes the minimum of inter-group AUC and intra-group AUC.}
\end{center}
Unlike usual discrimination-aware approaches that put constraints on the norm of fairness metrics, the maximin principle does not introduce unnecessary harm \citep{ustun2019fairness, martinez2020minimax}. Hence it is more natural for our initial purpose of learning via AUC maximization.

Our main contributions are summarized as follows. 

\begin{enumerate}

\item We justify the necessity of simultaneously achieving fairness in terms of intra- and inter-group AUCs. We then propose a minimax learning framework under the Rawlsian principle, collecting the objectives of both into one. 

\item We propose a stochastic algorithm that updates the model parameter by gradient descent  and the group weight by mirror ascent. We then prove the maximum re-weighted group error is guaranteed to be minimized by our algorithm in the nonconvex-concave setting.

\item We conduct numerical experiments on real-world datasets. We demonstrate the effectiveness of the minimax AUC fairness framework by observing    substantial utility improvement over the notion of equality of error rates, and also fairness improvement over frameworks dealing with intra-group or inter-group AUC alone.

\end{enumerate}

\subsection{Related Work}

Algorithmic fairness has received much attention recently, especially in the context of classification. There are three main strategies for bias mitigation: pre-processing the training data \citep{dwork2012fairness,feldman2015certifying,calmon2017optimized}, enforcing fairness directly during the training step (known as in-processing) \citep{kamishima2011fairness,agarwal2018reductions, donini2018empirical}, and post-processing the classifications \citep{hardt2016equality,pleiss2017fairness,lohia2019bias}. Our work fits in the middle category of in-processing.

\subsubsection{AUC-Related Fairness Metrics.} Several works have attempted to address unfairness concerns in AUC-related problems (See Appendix \ref{sec:literature} for a complete discussion). \citet{dixon2018measuring} propose the Pinned AUC fairness metric, which works by resampling the data such that each of the two groups make up 50\% of the data, and then calculating the AUC on the resampled dataset. \citet{beutel2019fairness} propose ranking pairwise fairness definitions and a methodology that regularizes the training objective through a term that measures the correlation between the residual of selected and unselected people. \citet{kallus2019fairness} observe inter-group AUC (xAUC) unfairness in the COMPAS dataset and propose a post-processing approach to achieve xAUC similarity. \citet{narasimhan2020pairwise} propose a cross-group fairness metric for general pairwise ranking problems and propose to maximize AUC under cross-group fairness constraints. 
\citet{vogel2021learning} define a fairness metric in terms of the ROC itself and propose a regularization method to achieve fairness. It is worth mentioning that all of the above works only focus on closing the gap of either intra-group metric or inter-group metric, but ignore the interplay between them. Furthermore, forcing unrealistic small group differences could harm the overall AUC maximization target.

\subsubsection{Minimax Principle in Algorithmic Fairness.} The Rawlsian principle has inspired several recent works to develop minimax frameworks to mitigate the disparate impact of machine learning models during training. In particular, \citet{mohri2019agnostic} apply the agnostic federated learning framework in the minimax group fairness context. \citet{martinez2020minimax} view the minimax fairness framework as a multi-objective function and pursue the Pareto front solution. \citet{lahoti2020fairness} study fairness without demographics and update group weights via adversarial learning. \citet{shekhar2021adaptive} propose an adaptive sampling algorithm based on the principle of optimism to update group weights. \citet{diana2021minimax} utilize minimax group errors as a fairness constraint upper bound for further optimization. 

The above minimax fairness frameworks all focus on classification and none has considered the AUC metric. Thus, the work herein is unique in applying the Rawlsian principle to AUC problems and considering both inter- and intra-group AUC simultaneously.  

\section{The Minimax Pairwise-Group Fairness Framework For Optimizing AUC \label{sec:problem}}

Let $X$ and $Y$ be two random variables. Here $Y \in \Ycal = \{\pm 1\}$ denotes the binary output label and  $X \in \Xcal$ denotes the input features where $\Xcal$ is a closed and bounded  domain in $\Rbb^d$. Define a  scoring function $f_\theta: \Xcal \rightarrow \Rbb$, where $\theta$ is the model parameter taking values in $\Theta$. 
The AUC of $f_\theta$ measures the probability that it correctly ranks a positive example above a negative example, i.e.,
\begin{equation}\label{eq:auc}
\auc(f_\theta) = \Ebb[\Ibb[f_\theta(X) > f_\theta(X')]| Y=1, Y'=-1],
\end{equation}
where $\Ibb$ is the indicator function taking value $1$ when the event is true, and $0$ otherwise. Observe that the definition \eqref{eq:auc} depends on  a pair of examples  $(X, Y)$ and $(X', Y')$.


In the context of fairness, we consider a third random variable $Z\in \Zcal$ to denote the sensitive group attribute, such as gender. For notational simplicity, we restrict our current interest to the case of two groups, i.e.,  $\Zcal = \{a, b\}$. However, our results can easily be extended to the general setting of  multiple groups. The distribution $\mathcal{D}$ of the triplet $(X, Y, Z)$ can be expressed as a mixture of the distributions of $X, Y|Z=z$. We define the group-level AUC as 
\begin{multline*}\label{eq:auc-group-level}
\auc_{z, z'}(f_\theta) = \\\Ebb[\Ibb[f_\theta(X) \!>\! f_\theta(X')]|
Y\!=\!1, Y'\!=\!-1, Z\!=\!z, Z'\!=\!z']. \numberthis   
\end{multline*}
It is worth mentioning the group-level AUC naturally enjoys a pairwise dependence w.r.t. groups $Z, Z'$. Following the terminology in \citet{beutel2019fairness}, we name such group-level AUC as \textit{intra-group} AUC  when $z=z'$, and \textit{inter-group} AUC when $z\neq z'$. Previous work has aimed for fairness in these definitions separately \citep{beutel2019fairness, kallus2019fairness, narasimhan2020pairwise}, i.e.\ either
\begin{align*}
\auc_{a, a}(f_\theta) & = \auc_{b, b}(f_\theta), \textrm{ or}\\
\auc_{a, b}(f_\theta) & = \auc_{b, a}(f_\theta).
\end{align*}
However, if one only requires fair treatment in terms of intra-group (resp. inter-group) AUCs, the model may still suffer from unfair treatment in terms of inter-group (resp. intra-group) AUCs as argued by \citet{kallus2019fairness} and Figure \ref{fig:limitation-inter}. To address this issue, one naive solution is to require fairness in both of them so that $\auc_{a, a}(f_\theta) \approx \auc_{b, b}(f_\theta) = \kappa_1$ and $\auc_{a, b}(f_\theta) \approx \auc_{b, a}(f_\theta) = \kappa_2$.  However, the potential discrepancy of $\kappa_1 \neq \kappa_2$ could also lead to unfairness. We elaborate in the following example.

Consider finding qualified candidates with gender as the sensitive attribute. In this scenario, AUC measures the probability of a qualified candidate being ranked higher than an unqualified candidate. Assuming the fair intra-group AUC is approximately 80\%, which indicates the chance of a qualified female ranking higher than an unqualified female is more or less the same 80\% as the male. Further assuming the fair inter-group AUC is approximately 60\%, which means the chance of a qualified female ranking higher than an unqualified male is more or less the same 60\% as the reverse case. Nevertheless, this leads to the unfair situation that a qualified female ranking higher than an unqualified male (60\%) is lower than herself ranking higher than another unqualified female (80\%)! Moreover, if male candidates are the majority of the applicant pool, her qualification is likely to be overwhelmed. 

The above example demonstrates that disparity between intra-group AUCs and inter-group AUCs should not be allowed. Ideally speaking, group-level AUC fairness should be cast as 
\begin{equation}\label{eq:equal-all}
\auc_{a, a}(f_\theta) \!=\! \auc_{a, b}(f_\theta) \!=\! \auc_{b, a}(f_\theta) \!=\! \auc_{b, b}(f_\theta).
\end{equation}
In the fair hiring example, the above identity can be interpreted as 
\begin{center}
\textit{The chance of a qualified candidate from any gender ranking higher than an unqualified candidate from any gender is the same.}
\end{center}
While the above discussion naturally leads to a regularization or a constrained optimization fairness scheme, we adapt the minimax fairness scheme \citep{martinez2020minimax, diana2021minimax}. To explain our choice, we first decompose the overall AUC as a mixture of the intra-group and the inter-group AUCs. 
\begin{multline*}\label{eq:auc-mixture-decomposition}
\auc(f_\theta) = \sum_{z \in \Zcal}\Pbb[Z=z, Z'=z] \auc_{z, z}(f_\theta)\\
+ \sum_{z \neq z'}\Pbb[Z=z, Z'=z'] \auc_{z, z'}(f_\theta), \numberthis
\end{multline*}
where $\Pbb[Z, Z']$ is the prior distribution of a pair of sensitive attributes. Recall that the target is to optimize AUC, or to maximize the probability of a qualified candidate ranking higher than an unqualifed one in the example. Eq. \eqref{eq:auc-mixture-decomposition} indicates that maximizing intra-group or inter-group AUC does not conflict with the interest of this target. Therefore, maximizing the minimum group-level AUC will have the potential of not significantly sacrificing the overall AUC while boosting the similarity on Eq. \eqref{eq:equal-all}. Based on the above discussion, we are in position to formalize the learning problem in a maximin sense.

\begin{equation}\label{eq:max-min-auc}
\max_{\theta \in \Theta}\min_{z,z' \in \Zcal} \auc_{z, z'}(f_\theta).  
\end{equation}
Eq. \eqref{eq:max-min-auc} is intractable for two reasons. First, the true distribution $\Dcal$ is unknown. In practice, we only have access to a sample $S = \{S_1, \cdots, S_n\}$ drawn from $\Dcal$. For convenience, we denote $S_i \in S$ by $\xbf_i^{z+}$ (resp. $\xbf_i^{z-}$) if it is coming from group $z$ with positive (resp. negative) label, we further denote $S^{z+} = \{\xbf_i^{z+}| i \in [n]\}$ and $S^{z-} = \{\xbf_i^{z-}| i \in [n]\}$ as two subsets of such examples. We can therefore define the empirical group-level AUC as:
\begin{multline*}
\widehat{\auc}_{z, z'}(f) = \widehat{\auc}(f; S^{z+}, S^{z'-})\\
= \frac{1}{n^{z+}n^{z'-}} \sum_{i=1}^{n^{z+}}\sum_{j=1}^{n^{z'-}} \Ibb[f_\theta(\xbf_i^{z+}) > f_\theta(\xbf_j^{z'-})],
\end{multline*}
where $n^{z+}$ and $n^{z'-}$ denote the number of samples from $S^{z+}, S^{z'-}$, respectively.  Furthermore, we replace the non-differentiable indicator function $\Ibb$ with some (sub)-differentiable and non-increasing surrogate loss function $\ell$. Let $\hat{R}^\ell_{z, z'} = \hat{R}^\ell(\,\cdot\,; S^{z+}, S^{z'-})$ denote the $\ell$-surrogated empirical risk for AUC with group $z, z'$. And let $\hat{R}^\ell(\cdot) = \hat{R}^\ell(\,\cdot\,; S) = (\hat{R}^\ell(\,\cdot\,; S^{z+}, S^{z'-}))_{z, z' \in \{a,b\}}$. The problem therefore becomes minimizing the largest group-level risk, which can be formulated as
\begin{equation}\label{eq:min-max-risk}
\min_{\theta \in \Theta} \max_{z, z' \in \Zcal} \hat{R}^\ell_{z, z'}(\theta).    
\end{equation}
The above problem is again non-differntiable. We introduce an additional variable $\lambda \in \Rbb^4$ and relax the problem as a zero-sum game between two players $\theta$ and $\lambda$. At each round, player $\theta$ intends to find a better position to minimizes the weighted risk, while player $\lambda$ assigns weights to each group that maximizes the weighted risk. Formally speaking, we arrive at the following minimax problem
\begin{equation}\label{eq:empirical-minimax-general}
\min_{\theta \in \Theta} \max_{\lambda \in \Lambda} F(\theta, \lambda) = \lambda^\top\hat{R}^\ell(\theta) = \sum_{z, z' \in \Zcal} \lambda_{z, z'}\hat{R}^\ell_{z, z'}(\theta),
\end{equation}
where $\Lambda = \{\lambda \in \Rbb^4| \sum_{z, z' \in \Zcal} \lambda_{z, z'} = 1, \lambda_{z, z'} \geq 0\}$ is a $2\times 2$-dimensional simplex. 

\section{Efficient Optimization Algorithm with Convergence Guarantee}

In this section, we introduce a stochastic gradient method summarized in Algorithm \ref{alg:rmw-fair-auc} to solve the minimax optimization problem \eqref{eq:empirical-minimax-general}.

\begin{algorithm}[!ht]
\caption{\it MinimaxFairAUC\label{alg:rmw-fair-auc}}
\begin{algorithmic}[1]
\STATE {\bf Inputs:} Training set $S$ with label $Y$ and protected attribute $Z$, model $f_\theta$, number of iterations $T$, batch size $m$, learning rates $\{\eta_\theta, \eta_\lambda\}$
\STATE {Initialize $\theta_0 \in 
\Theta$ and $\lambda_0 \in \Lambda$ with $\lambda_{z, z'} = \frac{n^{z+} n^{z'-}}{n^+ n^-}$ for all $z, z' \in \Zcal$}
\FOR{$t=1$ to $T-1$}
\STATE{$B_t = \texttt{StratifiedSampler}_m(S; Y, Z)$}
\STATE{$\theta_t = \theta_{t-1} -\eta_\theta \lambda_{t-1}^\top \nabla \hat{R}^\ell(\theta_{t-1}; B_t)$}
\STATE{$\gamma_t = \lambda_{t-1} \exp(\eta_\lambda \hat{R}^\ell(\theta_{t-1}; B_t))$}
\STATE{$\lambda_t = \gamma_t / \|\gamma_t\|_1$}
\ENDFOR
\STATE {\bf Outputs:} $\theta_\tau \sim \texttt{Unif}(\{\theta_t\}_{t=1}^T)$
\end{algorithmic}
\end{algorithm}
Let us illustrate Algorithm \ref{alg:rmw-fair-auc} in detail. At Line 4, the \texttt{StratifiedSampler} operator randomly sub-samples mini-batch by dividing the training set $S$ into strata based on the label and the group attribute. We note that this sub-sampling scheme will later guarantee the stochastic gradients being unbiased estimators of the full sample gradients, which is essential for the stochastic optimization analysis. We summarize this result in the following proposition.

\begin{proposition}\label{prop:stratified-sampling}
For any fixed $\theta \in \Theta$, let $B \subset S$ be given by \texttt{StratifiedSampler}. The following statement holds
\begin{align*}
\Ebb_B[\hat{R}^\ell(\theta; B)] = \hat{R}^\ell(\theta; S), \Ebb_B[\nabla\hat{R}^\ell(\theta; B)] = \nabla\hat{R}^\ell(\theta; S).
\end{align*}
\end{proposition}
The detailed description of \texttt{StratifiedSampler} and the proof of Proposition \ref{prop:stratified-sampling} are deferred to Appendix \ref{sec:stratified-sampling}. While a uniform sampler over the full dataset $S$ also can be shown to construct unbiased estimators, it can miss certain groups or labels during training, especially when the dataset is imbalanced in groups or labels \citep{shekhar2021adaptive}.

At Line 5, the player $\theta$ performs stochastic gradient descent with learning rate $\eta_\theta$. Here we slightly abuse the notation and use $\lambda_t$ to denote the $t$-th iteration of $\lambda \in \Rbb^4$. At Line 6-7, the player $\lambda$ performs the well-known exponential weight updates. We note that the exponential weight update is equivalent to mirror ascent when the mirror map $\Phi: \Lambda \rightarrow \Rbb$ is given by the negative entropy \citep{bubeck2015convex}, i.e. $\Phi(\lambda) \!=\! \sum_{z, z' \in \Zcal}\lambda_{z, z'} \log(\lambda_{z, z'})$. The exact update is given by
\begin{align*}
\nabla \Phi(\gamma_t) = & \nabla \Phi(\lambda_{t-1}) + \eta_t \hat{R}^\ell(\theta_{t-1}; B_t),\\
\lambda_t = & \arg\min_{\lambda \in \Lambda} D_\Phi(\lambda \parallel \gamma_t),
\end{align*}
where $D_\Phi(\lambda \parallel \lambda') = \Phi(\lambda) - \Phi(\lambda') - \nabla \Phi(\lambda')^\top (\lambda - \lambda')$ is the Bregman divergence associated with $\Phi$. Therefore, Algorithm \ref{alg:rmw-fair-auc} can be viewed as a stochastic gradient descent mirror ascent method. Next we introduce some standard assumptions for our convergence analysis.



\begin{assumption}\label{assmp:lipschitz-continuity}
For any $\theta \in \Theta$ and $\lambda \in \Lambda$, the gradients of $F$ are bounded by $G_\theta$ and $G_\lambda$ respectively, i.e. $\|\lambda^\top \nabla \hat{R}^\ell(\theta; S)\|_2 \leq G_\theta$, and $\|\hat{R}^\ell(\theta; S)\|_\infty \leq G_\lambda$.
\end{assumption}

\begin{assumption}\label{assmp:smoothness}
The objective $F$ is $L_\theta$ and $L_\lambda$ smooth respectively, i.e. $\|\lambda^\top\nabla \hat{R}^\ell(\theta; S) - \lambda^\top \nabla \hat{R}^\ell(\theta'; S)\|_2 \leq L_\theta \|\theta - \theta'\|_2$ and $\|\hat{R}^\ell(\theta; S) - \hat{R}^\ell(\theta'; S)\|_\infty \leq L_\lambda \|\lambda - \lambda'\|_1$ for any $\theta, \theta' \in \Theta$ and $\lambda, \lambda' \in \Lambda$. 
\end{assumption}

\begin{assumption}\label{assmp:bounded-variance}
For any fixed $\theta \in \Theta, \lambda \in \Lambda$ and randomly sampled pair $\xi = \{(x, y, z), (x',y', z')\} \!\subset\! S$, the variances of the stochastic gradients of the function $F(\cdot, \cdot; \xi)$ are bounded by $\sigma_\theta^2$ and $\sigma_\lambda^2$ respectively, i.e. $\Ebb_\xi[\|\lambda^\top\nabla \hat{R}^\ell(\theta; \xi) \!-\! \lambda^\top\nabla \hat{R}^\ell(\theta; S)\|_2^2] \!\leq\! \sigma_\theta^2$ and $\Ebb_\xi[\|\hat{R}^\ell(\theta; \xi) \!-\! \hat{R}^\ell(\theta; S)\|_\infty^2] \!\leq\! \sigma_\lambda^2$. 
\end{assumption}
Denote $G = \max\{G_\theta, G_\lambda\}$, $L = \max\{L_\theta, L_\lambda\}$ and $\sigma = \max\{\sigma_\theta, \sigma_\lambda\}$. Now let $P(\theta) = \max_{\lambda \in \Lambda} F(\theta, \lambda)$. Its Moreau envelope $P_{1/2L}$ is defined as $P_{1/2L}(\omega) = \min_{\theta} P(\theta) + L\|\theta - \omega\|_2^2$. With the above setup, we are in position to present the convergence of Algorithm \ref{alg:rmw-fair-auc}. 

\begin{theorem}[Informal]\label{thm:convergence}
Suppose Assumption \ref{assmp:lipschitz-continuity}, \ref{assmp:smoothness} and \ref{assmp:bounded-variance} hold true. Then the output $\theta_\tau$  of Algorithm \ref{alg:rmw-fair-auc} satisfies
\begin{equation}\label{eq:moreau-envelope-convergence}
\Ebb[\|\nabla P_{1/2L}(\theta_\tau)\|_2 \leq \epsilon(T, \eta_\theta, \eta_\lambda),
\end{equation}
where $\epsilon(T, \eta_\theta, \eta_\lambda)$ is an absolute constant. In particular, to achieve some small $\epsilon = \epsilon(T, \eta_\theta, \eta_\lambda)$, one choose $\eta_\theta = \Theta(\epsilon^4)$, $\eta_\lambda = \Theta(\epsilon^2)$ and $T = \Ocal(\epsilon^{-8})$.
Furthermore, there exists $\hat{\theta} \in \Theta$ such that $\Ebb[\|\hat{\theta} - \theta_\tau\|_2] \leq \epsilon / 2L$ and it satisfies
\begin{equation}\label{eq:primal-convergence}
\Ebb[\min_{\xi \in \partial P(\hat{\theta})} \|\xi\|_2 ]\leq \epsilon.
\end{equation}
\end{theorem}
The exact statement and detailed proof are deferred to Appendix \ref{sec:proof-convergence}. The proof of Theorem \ref{thm:convergence} is non-trivial and consists of several intermediate lemmas. The main idea is to derive a telescoping upper bound on the error term
\[\Delta_t = \Ebb_{B_t}[\max_{\lambda} F(\theta_t, \lambda) - F(\theta_t, \lambda_t)]
\]
by utilizing the concavity of $F(\theta_t, \cdot)$ and the Pythagorean inequality w.r.t.\ the Bergman divergence $D_\Phi$. It is worth noting that this result even holds for general nonconvex-concave problems over the simplex, which is of interest in its own right. We leave the study of the generalization error as future work. We end this section with several remarks on the implications of Theorem \ref{thm:convergence} and the comparison of Algorithm \ref{alg:rmw-fair-auc} with other minimax fairness algorithms.

\begin{remark}\label{rem:implication-convergence}
Eq. \eqref{eq:moreau-envelope-convergence} characterizes the local gradient convergence of the output $\theta_\tau$ on the Moreau envelope $P_{1/2L}$. Such a convergence bound can be transferred to the objective at concern $P$ in \eqref{eq:primal-convergence} \citep{davis2019stochastic}.  We call $P$ the primal objective as it has been taken out of the maximization on the dual variable $\lambda \in \Lambda$. In the ideal case, the group weights $\lambda$ only represent the maximum group risk. Therefore the convergence on the primal objective is exactly seeking the (local) minimum of the maximum group risk, defined as the original problem \eqref{eq:min-max-risk}. Furthermore, the primal objective also provides guidance on the stopping criterion in empirical evaluations, which is known to be vague for general minimax optimization. That is, to stop when the maximum group-level AUC risk is saturated.
\end{remark}

\begin{remark}\label{rem:comparison-minimax-algorithm}
One of the main focuses of minimax fairness algorithms is how to update the model parameter and the group weights. The most closely related algorithm by \citet{diana2021minimax} also applies the exponential weights update on the group weights. However, the model update in that algorithm requires the exact solution on the weighted group risks at each iteration. This is time-consuming, aside from being intractable in most cases. On the contrary, Algorithm \ref{alg:rmw-fair-auc} is iteration-efficient as it alternately updates the model parameter and the group weight based on the stochastic mini-batch. The stochastic gradient method in \citet{martinez2020minimax} relieves the intractability, yet still require a computationally-expensive inner loop to update the model parameter. Moreover, since the algorithm is heuristic, the authors provide no convergence analysis. \citet{mohri2019agnostic} also apply an efficient stochastic optimization algorithm by updating the group weights by stochastic gradient ascent. However, such an update is sub-optimal compared to our stochastic mirror ascent on simplex \citep{beck2003mirror}. Furthermore, their convergence relies on a convexity assumption while our results hold for the general nonconvex model class. Finally, it is worth noting the above minimax fairness frameworks all focus on classification and confusion matrix-based metrics; our algorithm is the first for group-level AUCs. 
\end{remark}



\section{Empirical Evaluations}

In this section, we evaluate the performance of Algorithm \ref{alg:rmw-fair-auc} in terms of utility and fairness.
\begin{figure*}[!ht]
\begin{subfigure}[!ht]{\linewidth}
    \centering
    \includegraphics[width=\linewidth]{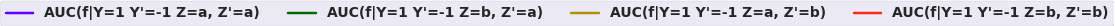}
\end{subfigure}%
    \hfill%
\begin{subfigure}[!ht]{0.24\linewidth}
    \includegraphics[width=\linewidth]{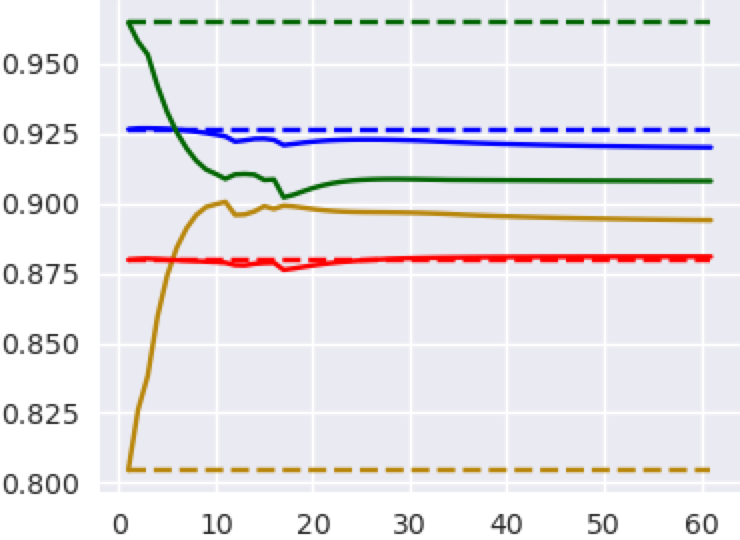}
\end{subfigure}
\hfill%
\begin{subfigure}[!ht]{0.24\linewidth}
    \includegraphics[width=\linewidth]{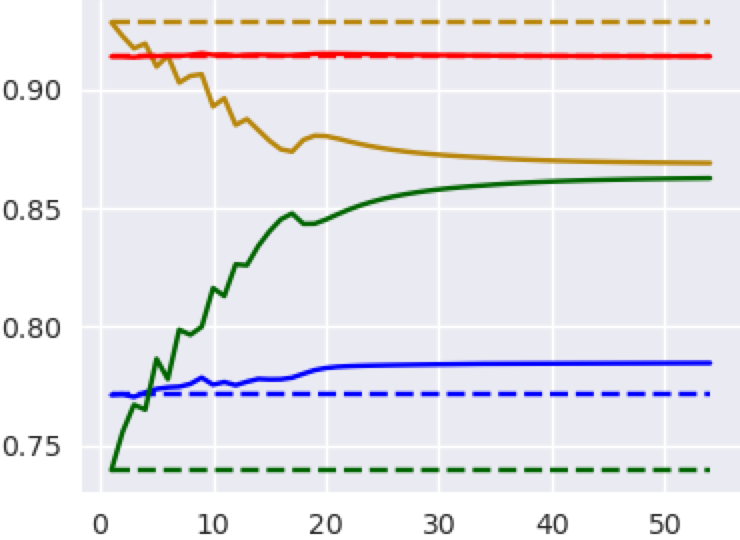}
\end{subfigure}
\hfill%
\begin{subfigure}[!ht]{0.24\linewidth}
    \includegraphics[width=\linewidth]{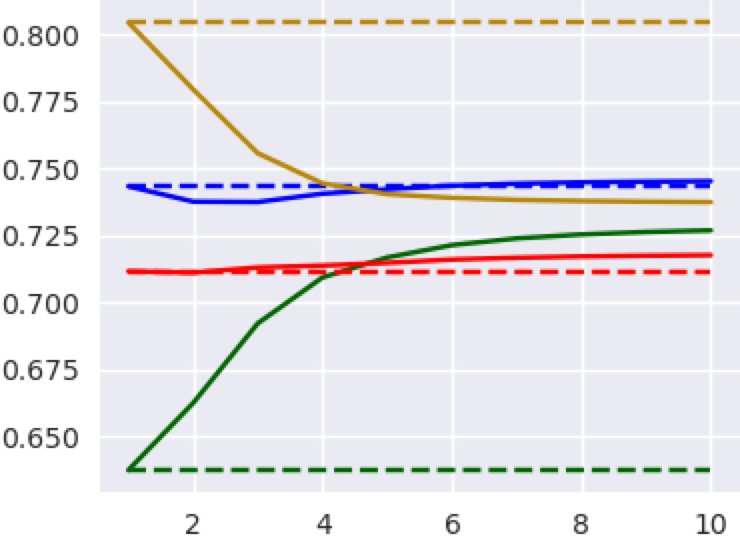}
\end{subfigure}
\hfill%
\begin{subfigure}[!ht]{0.24\linewidth}
    \includegraphics[width=\linewidth]{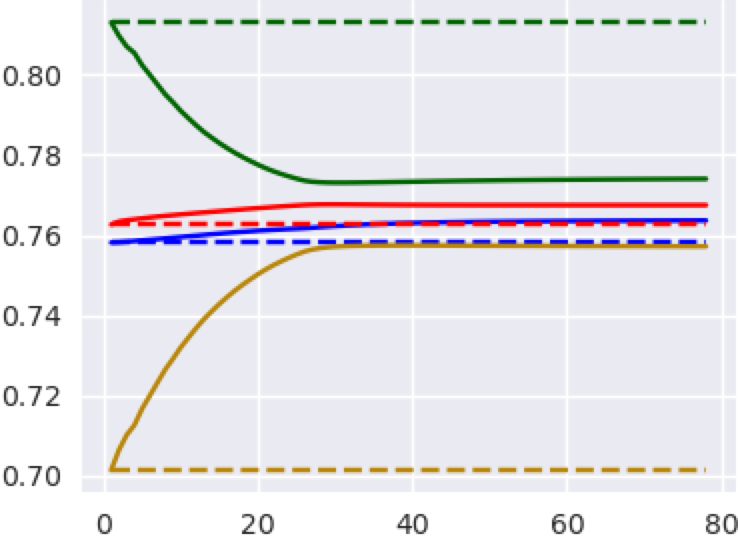}
\end{subfigure}
    \hfill%
\begin{subfigure}[!ht]{0.24\linewidth}
    \includegraphics[width=\linewidth]{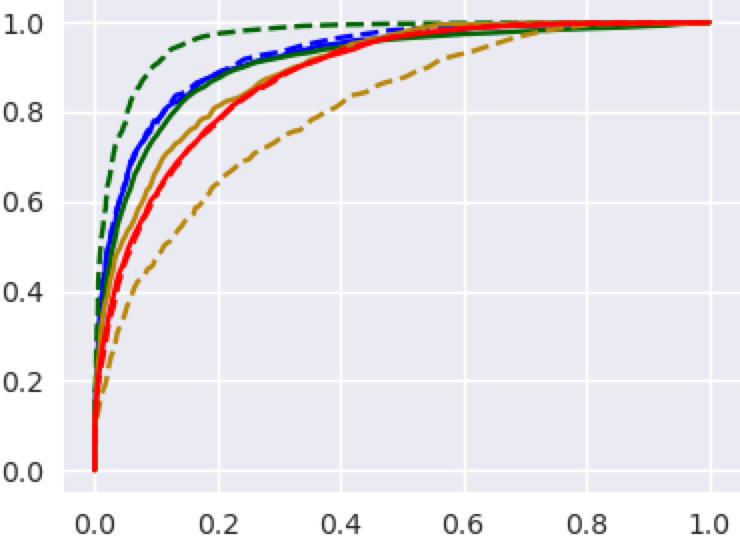}
    \caption{\it \texttt{Adult}}
\end{subfigure}
\hfill%
\begin{subfigure}[!ht]{0.24\linewidth}
    \includegraphics[width=\linewidth]{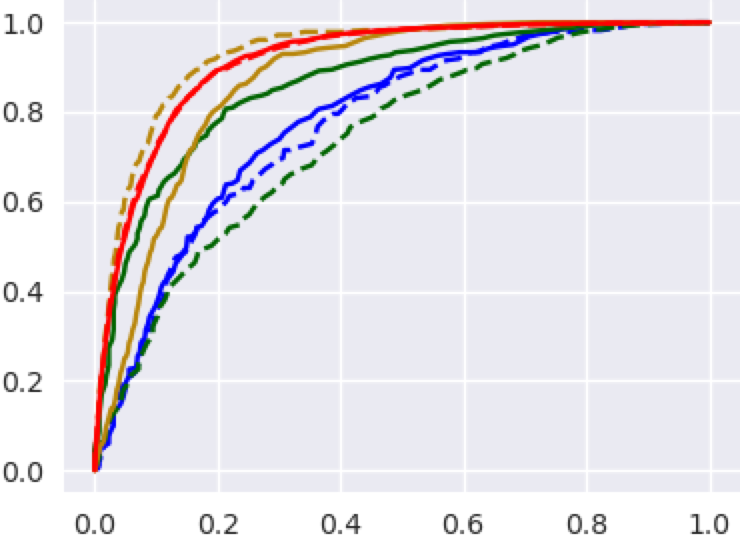}
    \caption{\it \texttt{Bank}}
\end{subfigure}
\hfill%
\begin{subfigure}[!ht]{0.24\linewidth}
    \includegraphics[width=\linewidth]{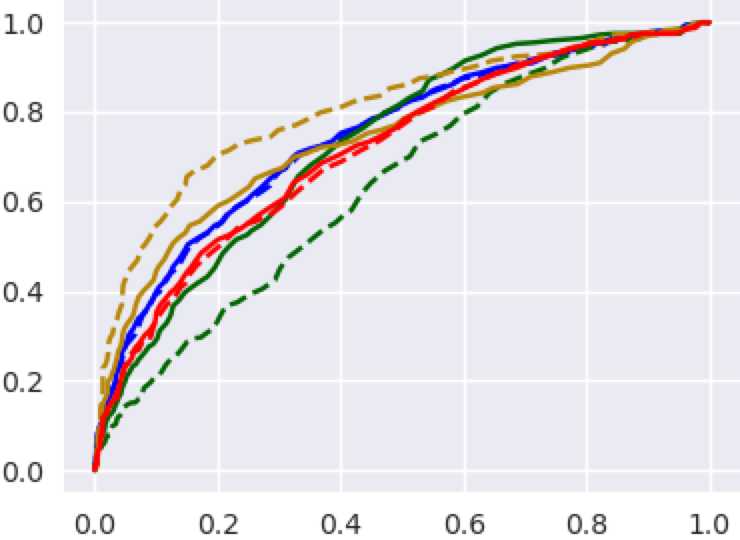}
    \caption{\it \texttt{COMPAS}}
\end{subfigure}
\hfill%
\begin{subfigure}[!ht]{0.24\linewidth}
    \includegraphics[width=\linewidth]{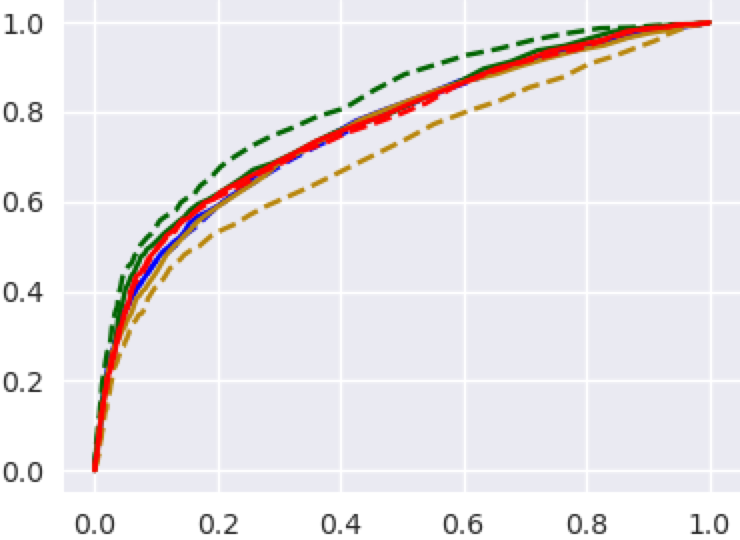}
    \caption{\it \texttt{Default}}
\end{subfigure}
\caption{\it Convergence plots on training set (upper half) and ROC plots on test set (lower half) of Algorithm \ref{alg:rmw-fair-auc} (solid curves) versus \texttt{AUCMax} (dashed curves). For convergence plots, the x-axis indicates the number of epochs, and the y-axis indicates the AUC score. For ROC plots, the x-axis indicates the FPR and the y-axis indicates the TPR. }
\label{fig:conv}
\vspace{-10pt}
\end{figure*}

\subsubsection{Datasets Information \& Feature Engineering.} We evaluate our algorithms on four datasets that have been commonly used in the fair machine learning literature \citep{zafar2017fairness, donini2018empirical}. We apply one-hot encoding to all categorical attributes and normalize all numerical attributes with zero mean and unit variance. The summary statistics of the datasets are given in Table \ref{tab:datasets}.

\begin{itemize}
    \item The \texttt{Adult} dataset is based on US census data and consists in predicting whether income exceeds $\$ 50K$ a year. The sensitive attribute is the gender of the individual, i.e. female ($Z=a$) or male ($Z=b$).
    \item The \texttt{Bank} dataset consists in predicting whether a client will subscribe to a term deposit. The sensitive attribute is the age of the individual: $Z=a$ when the age is less than 25 or over 60 and $Z=b$ otherwise.
    \item The \texttt{Compas} dataset consists in predicting recidivism of convicts in the US. The sensitive variable is the race of the individual, precisely $Z=a$ if the individual is categorized as non Caucasian and $Z=b$ if the individual is categorized as Caucasian.
    \item The \texttt{Default} dataset \citep{yeh2009comparisons} investigates customers’ default payments. The goal is to predict whether a customer will face the default situation in the next month or not. The sensitive attribute is the gender of the individual, i.e. female ($Z=a$) or male ($Z=b$).
\end{itemize}
\begin{table}[!ht]
\centering
\small
\setlength\tabcolsep{2.5pt}
\vspace{-10pt}
\begin{tabular}{|c|c|c|c|c|}
\hline
Name & \# instances & \# attributes & Group ratio & Class ratio \\\hline 
\texttt{Adult} & 48,842 & 15 & 0.48:1 & 3.03:1 \\\hline
\texttt{Bank} & 41,188 & 21 & 0.05:1 & 7.55:1 \\\hline
\texttt{Compas} & 11,757 & 53 & 1.86:1 & 1.94:1 \\\hline
\texttt{Default} & 30,000 & 24 & 1.52:1 & 3.52:1 \\\hline
\end{tabular}
\caption{\it Dataset Statistics. The group ratio is given by the protective attribute $Z=a$ versus $Z=b$. The Class ratio is given by negative versus positive class.}
\label{tab:datasets}
\vspace{-10pt}
\end{table}

\subsubsection{Choice of Models and Loss Functions.} To parameterize the family of scoring functions, we used a simple fully-connected neural network of 2 hidden layers with ReLU activation and batch normalization. The detail of the network is deffered in Appendix \ref{sec:exp}. The surrogate loss function $\ell$ is chosen as the logistic loss, i.e. $\ell: s \mapsto \log(1 + \exp(-s))$ since it is statistically consistent \citep{gao2015consistency} of the original 0/1 loss.

\subsubsection{Baselines.} We compare our framework with three in-processing baselines that have been proposed to 1) achieve fair group-level AUC scores; or 2) address unfair impact in terms of the Rawlsian principle.

\begin{itemize}
\item The \texttt{AUCMax} algorithm conducts AUC maximization on the full dataset without differentiating any groups. It updates the model parameter by mini-batch SGD. 
\item The \texttt{MinimaxFair} algorithm by \citet{diana2021minimax} aims to minimize the maximum group-level misclassification error. We replace the intractible model update by 10 epochs of SGD. 
\item  There are regularization methods \citep{beutel2019fairness} or constrained optimization methods \citep{narasimhan2020pairwise} targeting inter-group AUC fairness. We pick the one by \citet{vogel2021learning} as a representative since the authors considered the same datasets. \citeauthor{vogel2021learning} learn fair AUC scores by regularization on the difference of inter-group AUCs, which we refer to as \texttt{InterFairAUC}.
\item We also consider AUC maximization under the constraint of \eqref{eq:equal-all} as an alternative to our AUC fairness criterion under the name \texttt{EqualAUC}. See Appendix \ref{sec:eer-auc} for details.
\end{itemize}

\subsubsection{Implementation Details\footnote{\url{https://github.com/zhenhuan-yang/MinimaxFairAUC}.}} 
We partition the datasets to training, validation and testing in the ratio 60\%:20\%:20\%. The batch size $|B|$, initial stepsizes $\eta^\theta_0, \eta^\lambda_0$ and other hyperparameters are chosen based on the validation set. For Algorithm \ref{alg:rmw-fair-auc}, early stopping is implemented based on the maximum group loss over the validation set. We repeat each experiment in 25 runs across different random seeds and report the average result on the testing set. More details are given in Appendix \ref{sec:exp}.

\subsection{Results on Real Datasets}


We first investigate the convergence property of Algorithm \ref{alg:rmw-fair-auc}. We initialize the model parameters $\theta_0$ as the one trained from \texttt{AUCMax} to better illustrate the improvement over the baseline. As shown in Figure \ref{fig:conv} upper half, the inter-group AUC unfairness is prevalent in all four datasets, and intra-group AUC unfairness exists in \texttt{Adult} and \texttt{Bank}. When there is only inter-group AUC unfairness, Algorithm \ref{alg:rmw-fair-auc} mainly lifts the lowest inter-group AUC score. This can potentially benefit intra-group AUC (cf. \texttt{Compas} and \texttt{Default}.) When both intra-group and inter-group AUC unfairness exist, Algorithm \ref{alg:rmw-fair-auc} alternately lifts the lowest AUCs from inter-group then from intra-group, leading to a smaller gap in both of them separately. We next investigate the generalization/test performance of Algorithm \ref{alg:rmw-fair-auc}. In Figure \ref{fig:conv} lower half, the ROC curves of Algorithm \ref{alg:rmw-fair-auc} narrows the gaps of \texttt{AUCMax} towards the middle. 

We next develop a quantitative understanding of Algorithm \ref{alg:rmw-fair-auc} on its utility and fairness performance against other baselines in Table \ref{tab:comparison}. \texttt{MinimaxFair} struggles on both metrics. This is due to the objective differences: \texttt{MinimaxFair} is aiming at classification and the disparity in accuracy. Therefore it may overlook the disparity caused by group-level AUCs. \texttt{InterFairAUC} does not perform too well on  fairness as expected. This is because the regularization only focuses on the difference of inter-group AUCs and eventually the unfairness of intra-group AUCs will dominate the min/max ratio, especially on \texttt{Adult} and \texttt{Bank} where original gaps of intra-group AUCs are large. It only increase the overall AUC on \texttt{Bank} because the dataset is imbalanced. \texttt{EqualAUC} achieves competitive min/max ratios on all datasets. However it suffers from a large utility drop compared to the non-fairness intended baseline \texttt{AUCMax}. On the contrary, Algorithm \ref{alg:rmw-fair-auc} preserves the largest overall AUC scores. It achieves over 99\% on \texttt{Adult} and \texttt{Bank} compared to \texttt{AUCMax}, and achieves even higher score on \texttt{Compas} and \texttt{Default}. This phenomenon is consistent with our argument in \eqref{eq:max-min-auc} that the proposed minimax objective does not conflict with AUC maximization. The further improvement may due to the initialization via \texttt{AUCMax}. Using fairness aware retraining to boost utility has been observed in the literature recently, especially when the motivation is seeking minimax fairness or its analog \citep{globus-harris2022algorithmic}. In the meantime, for the fairness measurement, Algorithm \ref{alg:rmw-fair-auc} provides a competitive min/max ratio versus \texttt{EqualAUC}. Even on \texttt{Bank}, \texttt{EqualAUC} is not significantly more fair (p value $>0.01$) than Algorithm \ref{alg:rmw-fair-auc} while Algorithm \ref{alg:rmw-fair-auc} has a better overall AUC metric (p value $<0.001$). On \texttt{Compas}, Algorithm \ref{alg:rmw-fair-auc} even reaches the best overall score and the ratio simultaneously. Therefore it can be concluded that Algorithm \ref{alg:rmw-fair-auc} has advantages over \texttt{EqualAUC} in the sense of better trade-offs between overall AUC and fairness.

\begin{table*}[th!]
\centering
\setlength\tabcolsep{0.5pt}
{
\begin{tabular}{|c|cc|cc|cc|cc|}
\hline
& \multicolumn{2}{c|}{\texttt{Adult}} & \multicolumn{2}{c|}{\texttt{Bank}} & \multicolumn{2}{c|}{\texttt{Compas}}   & \multicolumn{2}{c|}{\texttt{Default}} \\ \hline 
\diagbox[width=7em, height=2em]{Algorithm}{Metric} & Overall & Min/Max & Overall & Min/Max & Overall & Min/Max & Overall & Min/Max \\ \hline
\texttt{AUCMax} & $\mathbf{.902\pm.002}$ & $.823\pm.005$ & $\underline{.910\pm.002}$ & $.780\pm.018$ & $.732\pm.004$ & $.779\pm.041$ & $.763\pm.005$ & $.871\pm.017$ \\ \hline
\texttt{MinimaxFair} & $.894\pm.007$ & $.905\pm.010$ & $.885\pm.003$ & $.827\pm.004$ & $.730\pm.001$ & $.913\pm.029$ & $.753\pm.002$ & $.909\pm.021$ \\ \hline
\texttt{InterFairAUC} & $.894\pm.004$ & $.950\pm.003$ & $\mathbf{.912\pm.001}$ & $.836\pm.018$ & $\underline{.738\pm.003}$ & $.939\pm.014$ & $\underline{.763\pm.003}$ & $.952\pm.024$ \\ \hline
\texttt{EqualAUC} & $.886\pm.003$ & $\underline{.953\pm.004}$ & $.866\pm.005$ & $\mathbf{.891\pm.025}$ & $.731\pm.003$ & $\underline{.956\pm.012}$ & $\underline{.761\pm.002}$ & $\mathbf{.972\pm.020}$ \\ \hline
\texttt{Algorithm \ref{alg:rmw-fair-auc}} & $\underline{.901\pm.004}$ & $\mathbf{.953\pm.002}$ & $.907\pm.004$ & $\underline{.858\pm.014}$ & $\mathbf{.741\pm.004}$ & $\mathbf{.961\pm.012}$ & $\mathbf{.767\pm.002}$ & $\underline{.968\pm.013}$ \\ \hline
\end{tabular}
\caption{\it Comparison of Algorithm \ref{alg:rmw-fair-auc} versus baselines. 'Overall' is the AUC score on the full dataset, measuring the utility. 'Min/Max' is the minimum group-level AUC score over the maximum one, measuring the fairness. The numbers are reported as 'Mean $\pm$ Standard Deviation'. Best results at each column are highlighted in \textbf{bold}. Second best are highlighted in \underline{underline}.}
\label{tab:comparison}
}
\vspace{-10pt}
\end{table*}



\begin{figure*}[!ht]
\includegraphics[width=\linewidth]{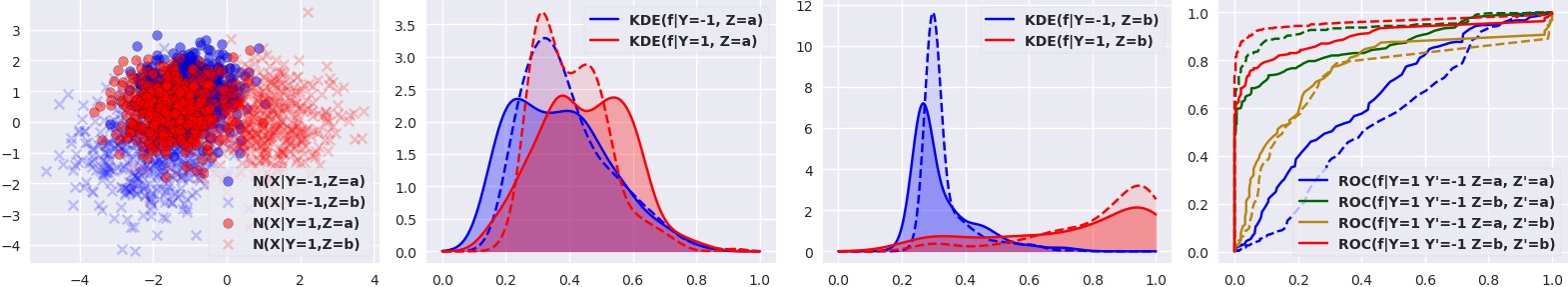}
\caption{\it Experiments on synthetic datasets. Dashed curves are baseline results via \texttt{AUCMax}.}
\label{fig:synthetic}
\vspace{-10pt}
\end{figure*}

\subsection{Synthetic Datasets Experiments}

\begin{figure}[!ht]
\begin{subfigure}{\linewidth}
    \includegraphics[width=\linewidth]{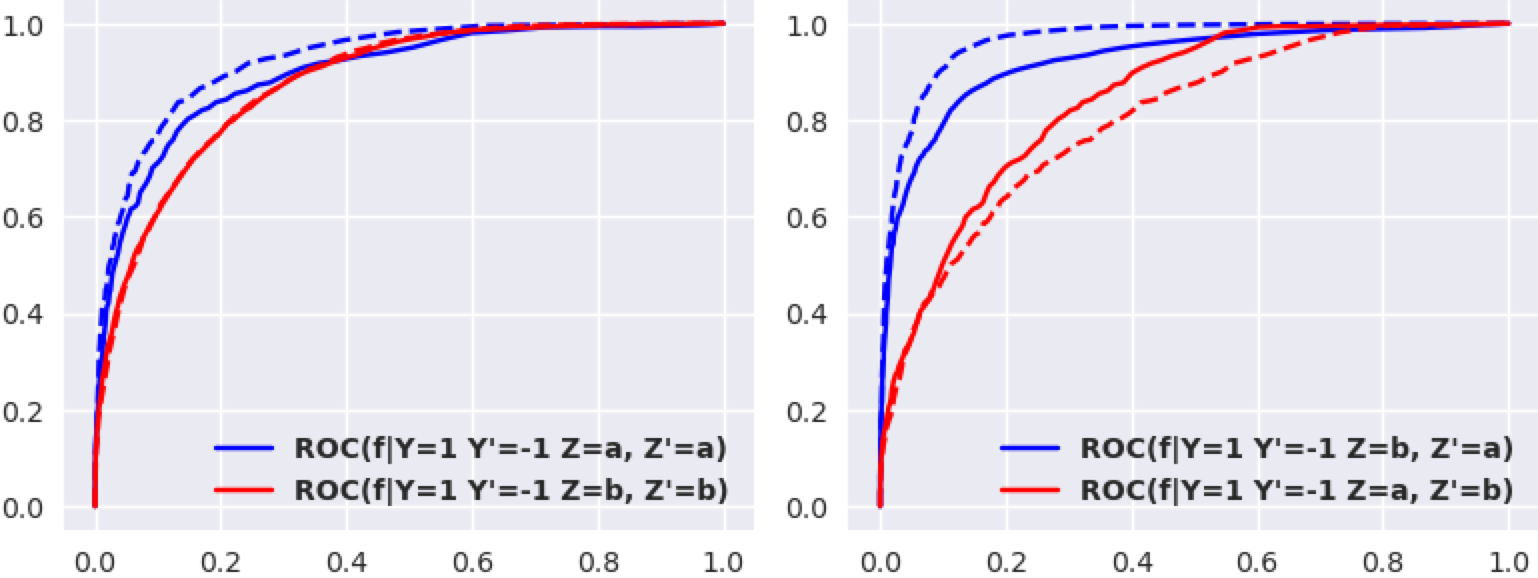}
    \caption{\it Intra-group AUCs only}
\end{subfigure}%
    \hfill%
\begin{subfigure}{\linewidth}
    \includegraphics[width=\linewidth]{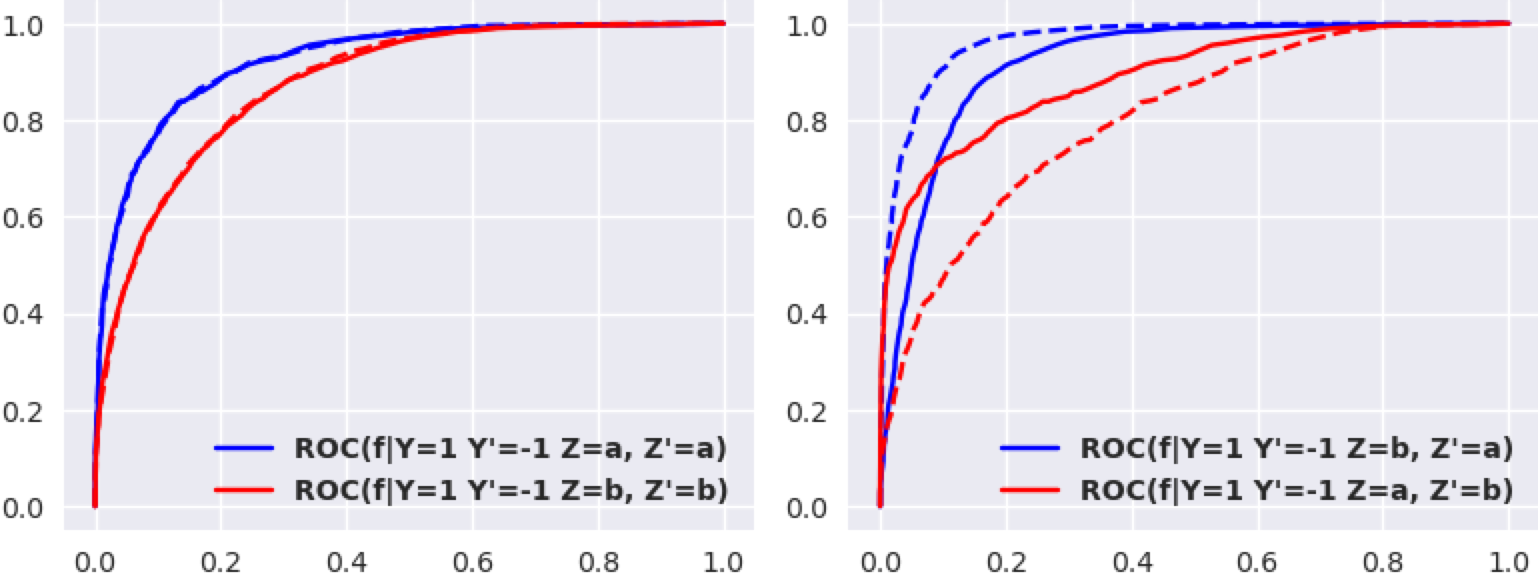}
    \caption{\it Inter-group AUCs only}
\end{subfigure}
\caption{\it Ablation study of Algorithm \ref{alg:rmw-fair-auc} (solid curves) versus \texttt{AUCMax} (dashed curves) on \texttt{Adult}.}
\label{fig:adult-ablation}
\vspace{-20pt}
\end{figure}

In this subsection, we design synthetic datasets to further understand the effectiveness of minimax AUC Algorithm \ref{alg:rmw-fair-auc} versus the unfair baseline. In particular, we generate two dimensional data points from different Gaussian distributions conditioned on the class label and the group attribute. See Appendix \ref{sec:exp} for more details. The generated data points are shown in the left most plot in Figure \ref{fig:synthetic}. We generate the same number of samples for each label and group partition so that the priors in \eqref{eq:auc-mixture-decomposition} can be ignored and all group-level AUCs contribute the same towards the overall AUC. Furthermore, we intentionally design the overlap between the positive and negative samples for group $a$ so that the scoring function faces difficulty differentiating them. As we see from the middle left plot in Figure \ref{fig:synthetic}, the baseline \texttt{AUCMax} learns almost nothing from group $a$ as the positive and negative KDEs are largely overlapped. Yet this fact does not stop the algorithm from maximizing the overall AUC as it keeps the negative KDE from group $b$ low. Therefore all three $AUC_{a,b}, AUC_{b,a}, AUC_{b,b}$ are optimized except $AUC_{a,a}$. However this is severely unfair for the positive sample from group $a$. Algorithm \ref{alg:rmw-fair-auc} mitigates this issue by separating the positive and negative KDEs from group $a$ to the correct direction. It also pushes the negative KDE from group $b$ to be lower so that positives from group $a$ maintain higher scores than negatives from group $b$. This is based on the sacrifice of the positives' scores from group $b$. Yet it may be acceptable as they are originally high enough.

\subsection{Ablation Studies}

In this subsection, we apply Algorithm \ref{alg:rmw-fair-auc} on the maximin AUC problem \eqref{eq:max-min-auc} with only intra-group AUCs or inter-group AUCs. We only report the results on \texttt{Adult} due to space limits. See Appendix \ref{sec:exp} for results on other datasets. In Figure \ref{fig:adult-ablation}(a), when intra-group AUCs are the only objective, Algorithm \ref{alg:rmw-fair-auc} slightly mitigates the intra-group AUC unfairness, yet it can potentially help mitigate the inter-group AUC unfairness as well. However, this improvement is very limited compared to lower half in Figure \ref{fig:conv}(a), where the unfairness is almost eliminated. In Figure \ref{fig:adult-ablation}(b), when inter-group AUCs are the only objective, inter-group fairness is more or less achieved with crossing of two curves, yet it is not guaranteed to mitigate the unfairness within intra-group AUCs.

\section{Conclusion}

In this paper, we propose a minimax learning framework for AUC maximization with fairness concern. Our framework addresses both intra-group and inter-group AUC unfairness as well as the discrepancy in between. Based on this framework, we design Algorithm \ref{alg:rmw-fair-auc}: an efficient algorithm with stochastic gradient descent on the model and mirror ascent on the group weights. We provide a non-trivial analysis and show that Algorithm \ref{alg:rmw-fair-auc} converges to the optimal solution in terms of the minimum group-level AUC in the nonconvex-concave setting. We conduct numerical experiments on both synthetic and real-world datasets to validate its utility and the fairness performance. One future direction is to consider fairness metrics involving group-level partial AUCs \citep{narasimhan2017support,yang2021all}. 

\section{Acknowledgement}
This work is supported by SUNY-IBM AI Research Alliance and NSF grants (IIS-2103450, IIS-2110546 and DMS-2110836).

\bibliography{reference}

\appendix

\onecolumn
\numberwithin{equation}{section}
\numberwithin{theorem}{section}
\numberwithin{figure}{section}
\numberwithin{table}{section}

\setcounter{secnumdepth}{1} 

\begin{center}
\textbf{\Large Appendix for ``Minimax  AUC Fairness: Efficient Algorithm with Provable Convergence"}
\end{center}

\section{AUC Related Fairness Metrics \label{sec:literature}}

In this section, we provide throughout discussion on the related work which define and address unfairness concerns in AUC-related problems in the literature. We apply the notation in the main text for consistency. 

\citet{borkan2019nuanced} is one of the first work which study the intra-group AUC fairness under the name of subgroup AUC fairness. \citet{dixon2018measuring} also target at similar intra-group AUC scores but is defined based on subsampling $B \subset S$, so that 
\[\widehat{\auc}(f_\theta; B^{z^+}, B^{z-}) = \widehat{\auc}(f_\theta; B^{z'+}, B^{z'-}),\]
where $|B^{z+}| + |B^{z-}| = |B^{z'+}| + |B^{z'-}|$ is required. For inter-group AUC fairness, \citet{kallus2019fairness} observe  the unfairness in this cross group definition. \citet{beutel2019fairness} define both of these pairwise fairness metrics in the task of pairwise ranking for recommendation systems, which are named intra-group pairiwse accuracy and inter-group pairwise accuracy originally in their work. In an independent work, \citet{narasimhan2020pairwise} also study intra-group and inter-group pairwise accuracy. It is worth noting that they briefly mention the fairness metric
\[\auc_{z, z'} = \kappa, \forall z,z'\in \Zcal,\]
which coincides with our Eq. \eqref{eq:equal-all} to a large extent. Such identity is named as pairwise equal opportunity in their work, as an analogy  of the equal opportunity for the binary classification task \citep{hardt2016equality}. 

Our work is  different from the work \citep{narasimhan2020pairwise} in the following aspects. Firstly, the motivation and application setting are unrelated. \citet{narasimhan2020pairwise} consider pairwise ranking  task whereas we focus on AUC maximization, which can be applied to in bipartite ranking and imbalanced classification tasks. Moreover, we provide a concrete application scenario (cf. fair hiring) to justify the necessity of debiasing all group-level AUCs. Secondly, our learning objective and theirs are not the same. \citet{narasimhan2020pairwise} integrates the discrepancy of group-level AUCs into their framework as either constrained optimization or robust optimization problem. This is because their fairness metric naturally requires an pre-defined "fairness level" $\kappa$, which could be difficult to determine before training/learning. On the other hand, our learning objective actually put both intra- and inter- group AUCs inside following the Rawlsian principle, which does not depend on the choice of "fairness level" and hence does not introduce additional harm towards the AUC maximization goal. Thirdly, the training procedure is different. \citeauthor{narasimhan2020pairwise} also apply a minimax framework to train their model. However, the minimax reformulation there is based on the introduction of Lagrange multipliers and it still solves the original constrained optimization problem, such minimax problem can have a vague stopping criterion especially when the model complexity is high (e.g. deep neural networks). As we argued in Remark \ref{rem:implication-convergence}, our convergence analysis provides a natural stopping criterion for the algorithm by looking at the saturation of the maximum group-level AUC, even for the nonconvex problem involving deep neural networks.  Finally, there are experimental differences between our work and  \citet{narasimhan2020pairwise}. For the ranking task with discrete groups, \citeauthor{narasimhan2020pairwise} only consider either inter-group AUC fairness or BNSP which we will introduce later. As we discussed in Figure \ref{fig:limitation-inter} and observed in our experiment, this treatment can overlook the concern of intra-group AUC fairness, whereas our approach can  handle both intra-group and inter-group AUC fairness as demonstrated in the experiments on the datasets of \texttt{Adult} and \texttt{Bank}.

There are  AUC related fairness metrics beyond intra-group and inter-group AUCs. \citet{borkan2019nuanced} propose 1) the Backgroup Negative Subgroup Positice (BNSP) AUC fairness,
\[
\widehat{\auc}(f_\theta; S^{z+}, S^-) = \widehat{\auc}(f_\theta; S^{z'+}, S^-)
\]
which enforces that the positive example from any group has the same probability of being ranked higher than a negative example, and 2) the Backgroup Positive Subgroup Negative (BPSN) AUC fairness,
\[
\widehat{\auc}(f_\theta; S^+, S^{z-}) = \widehat{\auc}(f_\theta; S^+, S^{z'-}),
\]
which can be viewed as the ranking extension of parity in false positive rates in binary classification \citep{hardt2016equality}. \citet{vogel2021learning} extend all the discussion above into ROC based fairness metrics. Note that group level TPR and FPR are given as 
\begin{align*}
\tpr_z(t, f_\theta) = & \Pbb[f_\theta(X) > t| Y=1, Z=z],\\
\fpr_z(t, f_\theta) = & \Pbb[f_\theta(X) > t| Y=-1, Z=z].
\end{align*}
The ROC curve is then defined as the plot of $\tpr(t, f_\theta)$ against $\fpr(t, f_\theta)$ for different values of $t$, i.e., for $t \in [0,1]$, group level ROC is given as
\[\text{ROC}_{z, z'}: t \mapsto \tpr_z\circ\fpr_{z'}^{-1}(t).
\] 
Consider the intra-group ROC fairness as an example. It can be similarly defined as the AUC case as follows
\[\text{ROC}_{a, a}(t) = \text{ROC}_{b, b}(t), \quad\forall t \in[0, 1].\]

\section{Stratified Random Sampling \label{sec:stratified-sampling}}

In this section, we introduce the \texttt{StratifiedSampler} used in Algorithm \ref{alg:rmw-fair-auc} and prove the unbiasedness of subsampled gradients in Proposition \ref{prop:stratified-sampling}. The \texttt{StratifiedSampler} is described in Algorithm \ref{alg:group-label-sampler}.

\begin{algorithm}[ht!]
\caption{Per Group and Label Stratified Sampler \label{alg:group-label-sampler}}
\begin{algorithmic}[1]
\STATE {\bf Inputs:} Dataset $S$ with label $Y$ and protected attribute $Z$, batch size $m$
\FOR{$z\in \{a, b\}$ and $y \in \{\pm 1\}$}
\STATE Uniformly sample without replacement $B^{zy}$ from $S^{zy}$ with size $m^{zy} \coloneqq \lceil m \cdot (n^{zy}/n)\rceil$
\ENDFOR
\STATE {\bf Outputs:} $B = \cup_{z, y} B^{zy}$
\end{algorithmic}
\end{algorithm}

The next proposition shows that the stochastic gradients constructed by Algorithm \ref{alg:group-label-sampler} are unbiased estimators of the full sample gradients. The proof is inspired by classical U-statistics theory \citep{clemenccon2008ranking}.

\begin{proposition}[Proposition \ref{prop:stratified-sampling} restated]
For any fixed $\theta \in \Theta$. Let $B \subset S$ be given by Algorithm \ref{alg:group-label-sampler}. The following statement holds
\begin{align*}
\Ebb_B[\hat{R}^\ell(\theta; B)] = & \hat{R}^\ell(\theta; S),\\
\Ebb_B[\nabla\hat{R}^\ell(\theta; B)] = & \nabla\hat{R}^\ell(\theta; S).
\end{align*}
\end{proposition}

\begin{proof}
The gradient w.r.t. $\lambda$ is given by
\[
(\partial_\lambda F(\theta, \lambda; S))_{z, z'} = \partial_{\lambda_{z, z'}}F(\theta, \lambda; S^{z+}\cup S^{z'-}) =  \frac{1}{n^{z+}} \frac{1}{n^{z'-}}\sum_{i=1}^{n^{z+}}\sum_{j=1}^{n^{z'-}}\ell(f_\theta(\xbf_i^{z+}) - f_\theta(\xbf_j^{z'-})),\, \forall z,z' \in \Zcal,
\]
and the gradient w.r.t. $\theta$ is given by
\[
\partial_\theta F(\theta, \lambda; S) = \sum_{z=1}^{k} \sum_{z'=1}^{k} \lambda_{z, z'}\partial_\theta\hat{R}^\ell(\theta;  S^{z+}\cup S^{z'-}) = \sum_{z=1}^{k} \sum_{z'=1}^{k} \frac{\lambda_{z, z'}}{n^{z+}n^{z'-}} \sum_{i=1}^{n^{z+}}\sum_{j=1}^{n^{z'-}}\nabla_\theta \ell(f_\theta(\xbf_i^{z+}) - f_\theta(\xbf_j^{z'-})).
\]
We first consider the proof on the gradient w.r.t. $\lambda$. In particular, the expectation of the stochastic gradient w.r.t. $\lambda$ is given by 
\begin{align*}
& \Ebb_{B \sim \texttt{Alg \ref{alg:group-label-sampler}}}[\partial_\lambda F(\theta, \lambda; B)]_{z, z'} = \Ebb_{B^{z+}, B^{z'-}}[\partial_{\lambda_{z, z'}} F(\theta, \lambda; B^{z+} \cup B^{z'-})] \\
= & \sum_{B^{z+} \subset S^{z+}} \sum_{B^{z'-} \subset S^{z'-}}\frac{1}{\binom{n^{z+}}{m^{z+}}} \frac{1}{\binom{n^{z'-}}{m^{z'-}}}\partial_{\lambda_{z, z'}} F(\theta, \lambda; B^{z+} \cup B^{z'-})\\
= & \sum_{B^{z+} \subset S^{z+}} \sum_{B^{z'-} \subset S^{z'-}}\frac{1}{\binom{n^{z+}}{m^{z+}}} \frac{1}{\binom{n^{z'-}}{m^{z'-}}} \frac{1}{m^{z+}} \frac{1}{m^{z'-}}\sum_{\xbf_i^{z^+} \in B^{z+}}\sum_{\xbf_j^{z'-} \in B^{z'-}}\ell(f_\theta(\xbf_i^{z+}) - f_\theta(\xbf_j^{z'-})) \\
= & \frac{1}{\binom{n^{z+}}{m^{z+}}} \frac{1}{\binom{n^{z'-}}{m^{z'-}}} \frac{1}{m^{z+}} \frac{1}{m^{z'-}} \sum_{\substack{\xbf_i^{z^+} \in S^{z+}\\\xbf_j^{z'-} \in S^{z'-}}} \sum_{\substack{B^{z+}\setminus\{\xbf_i^{z+}\} \subset S^{z+}\setminus\{\xbf_i^{z+}\} \\B^{z'-}\setminus\{\xbf_j^{z'-}\} \subset S^{z'-}\setminus\{\xbf_j^{z'-}\}}}\ell(f_\theta(\xbf_i^{z+}) - f_\theta(\xbf_j^{z'-}))\\
= & \frac{1}{\binom{n^{z+}}{m^{z+}}} \frac{1}{\binom{n^{z'-}}{m^{z'-}}} \frac{1}{m^{z+}} \frac{1}{m^{z'-}} \sum_{i=1}^{n^{z+}} \sum_{j=1}^{n^{z'-}} \binom{n^{z+}-1}{m^{z+}-1} \binom{n^{z'-}-1}{m^{z'-}-1} \ell(f_\theta(\xbf_i^{z+}) - f_\theta(\xbf_j^{z'-}))\\
= &  \frac{1}{n^{z+}} \frac{1}{n^{z'-}}\sum_{i=1}^{n^{z+}}\sum_{j=1}^{n^{z'-}}\ell(f_\theta(\xbf_i^{z+}) - f_\theta(\xbf_j^{z'-})), \, \forall z,z' \in \Zcal.
\end{align*}
The expectation of the stochastic gradient w.r.t. $\theta$ is given by 
\begin{align*}
\Ebb_{B \sim \texttt{Alg \ref{alg:group-label-sampler}}}[\partial_\theta F(\theta, \lambda; B)] = &  \Ebb_{B \sim \texttt{Alg \ref{alg:group-label-sampler}}}[ \sum_{z=1}^{k} \sum_{z'=1}^{k}\lambda_{z, z'} \partial_\theta\hat{R}^\ell(\theta; B^{z+} \cup B^{z'-})] \\
= & \sum_{z=1}^{k} \sum_{z'=1}^{k}\lambda_{z, z'} \Ebb_{B^{z+}, B^{z'-}}[\partial_\theta\hat{R}^\ell(\theta; B^{z+} \cup B^{z'-})] \\
= & \sum_{z=1}^{k} \sum_{z'=1}^{k} \frac{\lambda_{z, z'}}{n^{z+}n^{z'-}} \sum_{i=1}^{n^{z+}}\sum_{j=1}^{n^{z'-}}\nabla_\theta \ell(f_\theta(\xbf_i^{z+}) - f_\theta(\xbf_j^{z'-})).
\end{align*}
THe proof is complete by calling the definiton of $\hat{R}$. 
\end{proof}


\section{Convergence Analysis \label{sec:proof-convergence}}

In this section, we provide theoretical justification of Theorem \ref{thm:convergence}. Our analysis is inspired by the convergence of stochastic gradient descent ascent (SGDA) in the nonconvex-concave Euclidean space \citep{lin2019gradient}. We provide novel analysis on visualizing the exponential weights update as the gradient ascent in the $\ell_1$ setting.  
It is worth noting that $P$ is not necessarily differentiable even if $R$ is Lipschitz continuous and smooth. Fortunately, the following structural lemma shows that $P$ is weakly convex and Lipschitz continuous.

\begin{lemma}[\citep{lin2019gradient}]\label{lem:moreau-envelope-property}
Assume Assumption \ref{assmp:lipschitz-continuity} and \ref{assmp:smoothness} hold, the function $P(\cdot; S) = \max_{\lambda \in \Lambda_{k^2}} F(\cdot, \lambda; S)$ is $L$-weakly convex and $G$-Lipschitz with $\partial_\theta F(\theta, \hat{\lambda}(\theta); S) \in \partial P(\theta; S)$.
\end{lemma}  

For weakly convex function $P$, \citet{davis2019stochastic} proposed an alternative notion of stationarity based on the Moreau envolope, i.e. $P_{1/2L}(\theta) = \min_{\gamma} P(\gamma) + L\|\gamma - \theta\|_2^2$. Once the norm of the gradient of Moreau envelope is small, i.e. $\|\nabla P_{1/2L}(\theta)\|_2 \leq \epsilon$, one can show that $\theta$ is close to a point $\hat{\theta}$ which has at least one small subgradient.

\begin{lemma}[\citep{davis2019stochastic}]\label{lem:approximate-stationary-point}
If $\theta$ satisfies $\|\nabla P_{1/2L}(\theta)\|_2 \leq \epsilon$, there exists $\hat{\theta}$ such that $\min_{\xi \in \partial P(\hat{\theta})} \|\xi\|_2 \leq \epsilon$ and $\|\theta - \hat{\theta}\|_2 \leq \epsilon / 2L$.
\end{lemma}
The analysis of inexact nonconvex subgradient descent \citep{davis2019stochastic} implies the following descent inequality on $P_{1/2L}(\theta_t; S)$. The proof of this lemma is a direct application of Lemma D.3 in \citet{lin2019gradient} hence it is omitted. One can easily check the $\ell_1$ setting on the dual variable $\lambda$ does not hurt the definition and the descent lemma on $P_{1/2L}$, which is defined again in the Euclidean space.

\begin{lemma}[]\label{lem:descent-on-primal}
Assume Assumption \ref{assmp:lipschitz-continuity}, \ref{assmp:smoothness} and \ref{assmp:bounded-variance} hold, let $\Delta_t = \Ebb[P(\theta_t; S) - F(\theta_t, \lambda_t; S)]$. Then Algorithm \ref{alg:rmw-fair-auc} yields
\[
\Ebb[P_{1/2L}(\theta_t; S)] \leq \Ebb[P_{1/2L}(\theta_{t-1}; S)] + 2\eta_\theta L \Delta_{t-1} - \frac{\eta_\theta}{4} \Ebb[\|\nabla P_{1/2L}(\theta_{t-1}; S)\|_2^2] + \eta_\theta^2 L^2(G^2 + \sigma^2/m).
\]
\end{lemma}
Our key technical novelty is to show that the error term $\Delta_t$ in the above lemma enjoys a recursive relation. The proof's idea is based on the Pythagorean ineuqlaity w.r.t. the Bergman divergence $D_\Phi$ in the $\ell_1$ setting. 

\begin{lemma}\label{lem:descent-on-dual}
Assume Assumption \ref{assmp:lipschitz-continuity}, \ref{assmp:smoothness} and \ref{assmp:bounded-variance} hold. Let $\Delta_t = \Ebb_t[P(\theta_t) - F(\theta_t, \lambda_t)]$. Then the following statement holds true for any $\tau \leq t-1$,
\begin{multline*}
\Delta_{t-1} \leq \eta_\theta G\sqrt{G^2+\sigma^2/m}(2t-2\tau-1) + (\Ebb_t[F(\theta_t, \lambda_t)] - F(\theta_{t-1}, \lambda_{t-1}))\\
+  \frac{1}{\eta_\lambda}(D(\lambda^\star(\theta_\tau)\parallel\lambda_{t-1}) - \Ebb_t[D(\lambda^\star(\theta_\tau)\parallel\lambda_t)]) + \eta_\lambda^2\sigma^2/m
\end{multline*}
\end{lemma}

\begin{proof}
For any $\lambda \in \Lambda_{k^2}$, the convexity of $\Lambda_{k^2}$ and the update formula of $\lambda_t$ imply that
\[(\lambda - \lambda_t)^\top (\nabla \Phi(\lambda_t) - \nabla\Phi(\lambda_{t-1}) - \eta_\lambda \nabla_\lambda F(\theta_{t-1}, \lambda_{t-1}; B_t)) \geq 0.\]
By the definition of Bregman divergence, we have
\[
\eta_\lambda(\lambda - \lambda_t)^\top \nabla_\lambda F(\theta_{t-1}, \lambda_{t-1}; B_t) \leq D(\lambda \parallel \lambda_{t-1}) - D(\lambda \parallel \lambda_t) - D(\lambda_t \parallel \lambda_{t-1}).
\]
Rearranging the inequality yields that
\begin{multline*}
\eta_\lambda(\lambda - \lambda_{t-1})^\top \nabla_\lambda F(\theta_{t-1}, \lambda_{t-1}; B_t) + \eta_\lambda(\lambda_{t-1} - \lambda_t)^\top \nabla_\lambda F(\theta_{t-1}, \lambda_{t-1}; S)\\
\leq D(\lambda \parallel \lambda_{t-1}) - D(\lambda \parallel \lambda_t) - D(\lambda_t \parallel \lambda_{t-1}) + \eta_\lambda(\lambda_{t-1} - \lambda_t)^\top (\nabla_\lambda F(\theta_{t-1}, \lambda_{t-1}; S) - \nabla F(\theta_{t-1}, \lambda_{t-1}; B_t)).
\end{multline*}
Using the Young's inequality and the $1$-strong convexity of $\Phi$, we have
\begin{align*}
& \eta_\lambda(\lambda_{t-1} - \lambda_t)^\top (\nabla_\lambda F(\theta_{t-1}, \lambda_{t-1}; S) - \nabla F(\theta_{t-1}, \lambda_{t-1}; B_t))\\
\leq & \frac{\|\lambda_t - \lambda_{t-1}\|_1^2}{4} + \eta_\lambda^2\|\nabla_\lambda F(\theta_{t-1}, \lambda_{t-1}; S) - \nabla F(\theta_{t-1}, \lambda_{t-1}; B_t)\|_\infty^2\\
\leq & \frac{D(\lambda_t\parallel\lambda_{t-1})}{2} + \eta_\lambda^2\|\nabla_\lambda F(\theta_{t-1}, \lambda_{t-1}; S) - \nabla F(\theta_{t-1}, \lambda_{t-1}; B_t)\|_\infty^2.
\end{align*}
Taking an expectation of both sides of the above inequality, conditioned on $(\theta_{t-1}, \lambda_{t-1})$, together with Assumption \ref{assmp:bounded-variance} yields that
\begin{multline*}\label{eq:pythagorean-by-bregman}
\eta_\lambda(\lambda - \lambda_{t-1})^\top \nabla_\lambda F(\theta_{t-1}, \lambda_{t-1}; S) + \Ebb_t[ \eta_\lambda(\lambda_{t-1} - \lambda_t)^\top \nabla_\lambda F(\theta_{t-1}, \lambda_{t-1}; S)] \\
\leq D(\lambda \parallel \lambda_{t-1}) - \Ebb_t[D(\lambda \parallel \lambda_t)] - \frac{1}{2}\Ebb_t[D(\lambda_t \parallel \lambda_{t-1})]  + \frac{\eta_\lambda^2\sigma^2}{m} \numberthis
\end{multline*}
Since $F(\theta_{t-1}, \cdot; S)$ is concave, we have
\[F(\theta_{t-1}, \lambda; S) \leq F(\theta_{t-1}, \lambda_{t-1}; S) + (\lambda - \lambda_{t-1})^\top \nabla_\lambda F(\theta_{t-1}, \lambda_{t-1}; S).\]
Since $F(\theta_{t-1}, \cdot; S)$ is $L$-smooth w.r.t. $\|\cdot\|_1$ and $\Phi$ is $1$-strongly convex w.r.t. $\|\cdot\|_1$, we have
\begin{align*}
- F(\theta_{t-1}, \lambda_t; S) \leq & - F(\theta_{t-1}, \lambda_{t-1}; S) - (\lambda_t - \lambda_{t-1})^\top \nabla_\lambda F(\theta_{t-1}, \lambda_{t-1}; S) + \frac{L}{2}\|\lambda_t - \lambda_{t-1}\|_1^2 \\
\leq & - F(\theta_{t-1}, \lambda_{t-1}; S) - (\lambda_t - \lambda_{t-1})^\top \nabla_\lambda F(\theta_{t-1}, \lambda_{t-1}; S) + L D(\lambda_t\parallel\lambda_{t-1}).
\end{align*}
Summing up the above two inequalities and taking conditional expectation on both sides, we have
\begin{multline*}
F(\theta_{t-1}, \lambda; S) - \Ebb_t[F(\theta_{t-1}, \lambda_t; S)] \leq (\lambda - \lambda_{t-1})^\top \nabla_\lambda F(\theta_{t-1}, \lambda_{t-1}; S)\\
+ \Ebb_t[(\lambda_{t-1} - \lambda_t)^\top \nabla_\lambda F(\theta_{t-1}, \lambda_{t-1}; S)] + L\Ebb_t[ D(\lambda_t\parallel\lambda_{t-1})].    
\end{multline*}
By picking $\eta_\lambda \leq \frac{1}{2L}$ and putting Eq. \eqref{eq:pythagorean-by-bregman} into the above inequality, we have
\begin{equation}\label{eq:dual-telescoping}
F(\theta_{t-1}, \lambda) - \Ebb_t[F(\theta_{t-1}, \lambda_t)] \leq \frac{1}{\eta_\lambda}(D(\lambda\parallel\lambda_{t-1}) - \Ebb_t[D(\lambda\parallel\lambda_t)]) + \frac{\eta_\lambda^2\sigma^2}{m}.
\end{equation}
Plugging $\lambda = \lambda^\star(\theta_\tau)$ ($\tau \leq t - 1$) in the above inequality yields that
\[F(\theta_{t-1}, \lambda^\star(\theta_\tau)) - \Ebb_t[F(\theta_{t-1}, \lambda_t)] \leq \frac{1}{\eta_\lambda}(D(\lambda^\star(\theta_\tau)\parallel\lambda_{t-1}) - \Ebb_t[D(\lambda^\star(\theta_\tau)\parallel\lambda_t)]) + \frac{\eta_\lambda^2\sigma^2}{m}.\]
By the definition of $\Delta_{t-1}$, we have
\begin{multline*}\label{eq:intermediate-bound}
\Delta_{t-1} \leq (F(\theta_{t-1}, \lambda^\star(\theta_{t-1})) - F(\theta_{t-1}, \lambda^\star(\theta_\tau))) + (\Ebb_t[F(\theta_t, \lambda_t)] - F(\theta_{t-1}, \lambda_{t-1})) + (\Ebb_t[F(\theta_{t-1}, \lambda_t)] - \Ebb_t[F(\theta_t, \lambda_t)]) \\
+ \frac{1}{\eta_\lambda}(D(\lambda^\star(\theta_\tau)\parallel\lambda_{t-1}) - \Ebb_t[D(\lambda^\star(\theta_\tau)\parallel\lambda_t)]) + \frac{\eta_\lambda^2\sigma^2}{m}. \numberthis
\end{multline*}
By the optimality of $\lambda^\star(\theta_\tau)$, we have $F(\theta_\tau, \lambda^\star(\theta_\tau)) \geq F(\theta_\tau, \lambda)$ for all $\lambda \in \Lambda_{k^2}$. It implies that
\begin{align*}
F(\theta_{t-1}, \lambda^\star(\theta_{t-1})) - F(\theta_{t-1}, \lambda^\star(\theta_\tau)) = & F(\theta_{t-1}, \lambda^\star(\theta_{t-1})) - F(\theta_\tau, \lambda^\star(\theta_{t-1})) + F(\theta_\tau, \lambda^\star(\theta_{t-1})) - F(\theta_{t-1}, \lambda^\star(\theta_\tau))\\
\leq & F(\theta_{t-1}, \lambda^\star(\theta_{t-1})) - F(\theta_\tau, \lambda^\star(\theta_{t-1})) + F(\theta_\tau, \lambda^\star(\theta_\tau)) - F(\theta_{t-1}, \lambda^\star(\theta_\tau)).
\end{align*}
By the $G$-Lipschitz continuity of $F(\cdot, \lambda)$ and Assumption \ref{assmp:bounded-variance}, we have
\begin{align*}
F(\theta_{t-1}, \lambda^\star(\theta_{t-1})) - F(\theta_\tau, \lambda^\star(\theta_{t-1})) \leq & G\Ebb_t[\|\theta_{t-1} - \theta_\tau\|_2]\leq \eta_\theta G\sqrt{G^2 + \sigma^2/m}(t-1-\tau),\\
F(\theta_\tau, \lambda^\star(\theta_\tau)) - F(\theta_{t-1}, \lambda^\star(\theta_\tau)) \leq & G\Ebb_t[\|\theta_\tau - \theta_{t-1}\|_2]\leq \eta_\theta G\sqrt{G^2 + \sigma^2/m}(t-1-\tau),\\
\Ebb_t[F(\theta_{t-1}, \lambda_t) - F(\theta_t, \lambda_t)] \leq & G\Ebb_t[\|\theta_{t-1} - \theta_t\|_2] \leq \eta_\theta G\sqrt{G^2 + \sigma^2/m}. 
\end{align*}
Putting theses pieces together and taking expectation on both sides of Eq. \eqref{eq:intermediate-bound} yields the desired inequality.
\end{proof}

Based on the recursive relation in Lemma \ref{lem:descent-on-dual}, the next lemma provides an upper bound for $\frac{1}{T+1}\sum_{t=0}^T\Delta_t$ using both telescoping and localization technique.

\begin{lemma}\label{lem:error-on-dual}
Let $B \leq T + 1$ and assume $(T+1)/B$ is an integer. The following statement holds true
\[\frac{1}{T+1}\sum_{t=0}^T\Delta_t \leq \eta_\theta G\sqrt{G^2 + \sigma^2/m}(B+1) + \frac{\widehat{\Delta}_0}{T+1} + \frac{D^2}{\eta_\lambda B} + \frac{\eta_\lambda\sigma^2}{m}.\]
\end{lemma}

\begin{proof}
We divide $\{\Delta_t\}_{t=0}^T$ into several blocks in which each block contains at most $B$ terms, given by
\[\{\Delta_t\}_{t=0}^{B-1}, \{\Delta_t\}_{t=B}^{2B-1}, \cdots, \{\Delta_t\}_{t=T-B+1}^T.\]
Then we have
\begin{equation}\label{eq:block-sum-1}
\frac{1}{T+1}\sum_{t=0}^T\Delta_t \leq \frac{B}{T+1}\Big(\sum_{i=0}^{(T+1)/B-1} \Big(\frac{1}{B}\sum_{t=iB}^{(i+1)B-1} \Delta_t\Big)\Big).
\end{equation}
Furthermore, letting $\tau=iB$ in the inequality in Lemma \ref{lem:descent-on-dual} yields that
\begin{align*}
\sum_{t=iB}^{(i+1)B - 1}\Delta_t \leq & \eta_\theta G\sqrt{G^2+ \sigma^2/m}B^2 + \frac{\eta_\lambda^2 \sigma^2B}{m} + \frac{1}{\eta_\lambda}D(\lambda^\star(\theta_{iB})\parallel\lambda_{iB}) + (\Ebb[F(\theta_{(i+1)B}, \lambda_{(i+1)B})] - \Ebb[F(\theta_{iB}, \lambda_{iB})])\\
\leq & \eta_\theta G\sqrt{G^2+ \sigma^2/m}B^2 + \frac{\eta_\lambda^2 \sigma^2B}{m} + \frac{D^2}{\eta_\lambda} + (\Ebb[F(\theta_{(i+1)B}, \lambda_{(i+1)B})] - \Ebb[F(\theta_{iB}, \lambda_{iB})])
\end{align*}
Plugging the above inequality into Eq. \eqref{eq:block-sum-1} yields
\begin{equation}\label{eq:block-sum-2}
\frac{1}{T+1}\sum_{t=0}^T\Delta_t \leq \eta_\theta G\sqrt{G^2+ \sigma^2/m}B + \frac{\eta_\lambda^2 \sigma^2}{m}  + \frac{D^2}{\eta_\lambda B} + \frac{\Ebb[F(\theta_{T+1}, \lambda_{T+1})] - \Ebb[F(\theta_0, \lambda_0)]}{T+1}.
\end{equation}
Since $F(\cdot, \lambda)$ is $G$-Lipschitz continuous, we have
\begin{align*}\Ebb[F(\theta_{T+1}, \lambda_{T+1})] - \Ebb[F(\theta_0, \lambda_0)] \leq & \Ebb[F(\theta_{T+1}, \lambda_{T+1})] - \Ebb[F(\theta_0, \lambda_{T+1})] + \Ebb[F(\theta_0, \lambda_{T+1})] - \Ebb[F(\theta_0, \lambda_0)]\\
\leq & \eta_\theta G\sqrt{G^2+\sigma^2/m}(T+1) + \widehat{\Delta}_0.
\end{align*}
Plugging the above inequality into Eq. \eqref{eq:block-sum-2} yields the desired inequality.
\end{proof}

Now we are at the position to formally state Theorem \ref{thm:convergence}. In particular, 

\begin{theorem}[Theorem \ref{thm:convergence} restated]
Let Assumption \ref{assmp:lipschitz-continuity}, \ref{assmp:smoothness}, \ref{assmp:bounded-variance} hold true. Let the step sizes be chosen as $\eta_\theta = \Theta(\epsilon^4 / (L^3 D^2 (G^2 + \sigma^2/m)))$ and $\eta_\lambda = \Theta(\epsilon^2/(L\sigma^2/m))$. Then, after $T$ iteration, the output $\theta_\tau$ of Algorithm \ref{alg:rmw-fair-auc} satisfies The iteration complexity of Algorithm \ref{alg:rmw-fair-auc} to return an $\epsilon$-stationary point is bounded by
\[
\Ebb[\|\nabla P_{1/2L}(\theta_\tau; S)\|_2^2] \leq \frac{4\widehat{\Delta}_P}{\eta_\theta(T+1)} + \frac{8L\widehat{\Delta}_0}{T+1} + \frac{3\epsilon^2}{4}.
\]
Furthermore, there exists $\hat{\theta}$ such that $\min_{\xi \in \partial P(\hat{\theta})} \|\xi\|_2 \leq \epsilon$ and $\|\theta_\tau - \hat{\theta}\|_2 \leq \epsilon / 2L$ by picking 
\[T = \Ocal\Big(\Big(\frac{L^3(G^2 + \sigma^2/m) D^2 \widehat{\Delta}_P}{\epsilon^4} + \frac{L^3D^2\widehat{\Delta}_0}{\epsilon^2}\Big)\max\Big\{1, \frac{\sigma^2}{\epsilon^2}\Big\}\Big).\]
\end{theorem}

\begin{proof}[Proof of Theorem \ref{thm:convergence}]
Summing up the inequality in Lemma \ref{lem:descent-on-primal} over $t=1, 2, \cdots T+1$ yields that
\[\Ebb[P_{1/2L}(\theta_{T+1}; S)] \leq \Ebb[P_{1/2L}(\theta_0; S)] + 2\eta_\theta L \sum_{t=0}^T\Delta_t - \frac{\eta_\theta}{4} \sum_{t=0}^T\Ebb[\|\nabla P_{1/2L}(\theta_t; S)\|_2^2] + \eta_\theta^2 L^2(G^2 + \sigma^2/m)(T+1).\]
Combining the above inequality with the inequality in Lemma \ref{lem:error-on-dual} yields that
\begin{multline*}
\Ebb[P_{1/2L}(\theta_{T+1}; S)] \leq \Ebb[P_{1/2L}(\theta_0; S)] + 2\eta_\theta L (T+1) \Big(\eta_\theta G\sqrt{G^2 + \sigma^2/m}(B+1) + \frac{D^2}{\eta_\lambda B} + \frac{\eta_\lambda\sigma^2}{m}\Big) + 2\eta_\theta L\widehat{\Delta}_0\\
- \frac{\eta_\theta}{4} \sum_{t=0}^T\Ebb[\|\nabla P_{1/2L}(\theta_t; S)\|_2^2] + \eta_\theta^2 L^2(G^2 + \sigma^2/m)(T+1).    
\end{multline*}
By the definition of $\widehat{\Delta}_P$, we have
\[
\frac{1}{T+1}\sum_{t=0}^T\Ebb[\|\nabla P_{1/2L}(\theta_t; S)\|_2^2] \leq \frac{4\widehat{\Delta}_P}{\eta_\theta(T+1)} + 8L \Big(\eta_\theta G\sqrt{G^2 + \sigma^2/m}(B+1) + \frac{D^2}{\eta_\lambda B} + \frac{\eta_\lambda\sigma^2}{m}\Big) + \frac{8L\widehat{\Delta}_0}{T+1} + 4\eta_\theta L^2(G^2 + \sigma^2/m).    
\]
Letting $B = \frac{D}{2}\sqrt{\frac{1}{\eta_\theta\eta_\lambda G\sqrt{G^2+\sigma^2/m}}}$, we have
\[
\frac{1}{T+1}\sum_{t=0}^T\Ebb[\|\nabla P_{1/2L}(\theta_t; S)\|_2^2] \leq \frac{4\widehat{\Delta}_P}{\eta_\theta(T+1)} + 16LD \sqrt{\frac{\eta_\theta G\sqrt{G^2 + \sigma^2/m}}{\eta_\lambda}} + \frac{8L\widehat{\Delta}_0}{T+1} + \frac{8\eta_\lambda L^2\sigma^2}{m} + 4\eta_\theta L^2(G^2 + \sigma^2/m).    
\]
By picking $\eta_\theta = \min\big\{\frac{\epsilon^2}{16L(G^2+\sigma^2/m)}, \frac{\epsilon^4}{8192L^3 D^2 G\sqrt{G^2 + \sigma^2/m}}, \frac{\epsilon^6}{65536L^3 D^2(\sigma^2/m) G\sqrt{G^2+\sigma^2/m}}\big\}$ and $\eta_\lambda = \min\big\{\frac{1}{2L}, \frac{\epsilon^2}{16L^2\sigma^2/m}\big\}$, we have
\[
\frac{1}{T+1}\sum_{t=0}^T\Ebb[\|\nabla P_{1/2L}(\theta_t; S)\|_2^2] \leq \frac{4\widehat{\Delta}_P}{\eta_\theta(T+1)} + \frac{8L\widehat{\Delta}_0}{T+1} + \frac{3\epsilon^2}{4}.
\]
The above inequality yields the first claim in the theorem by taking the expectation of $\theta_\tau$ due to the uniform sampling. Furthermore, by Lemma \ref{lem:approximate-stationary-point}, it implies that the number of iterations required by Algorithm \ref{alg:rmw-fair-auc} to return an $\epsilon$-stationary point is bounded by
\[\Ocal\Big(\Big(\frac{L(G^2 + \sigma^2/m) \widehat{\Delta}_P}{\epsilon^4} + \frac{L\widehat{\Delta}_0}{\epsilon^2}\Big)\max\Big\{1, \frac{L^2D^2}{\epsilon^2}, \frac{L^2 D^2 \sigma^2}{\epsilon^4}\Big\}\Big),\]
which gives the same total gradient complexity in the second claim.
\end{proof}

\section{Equal Error Rates for AUC \label{sec:eer-auc}} 

In this section, we describe the equal error rates framework mentioned in empirical evaluations for achieving intra-group and inter-group AUC fairness as in Eq. \eqref{eq:equal-all} and . The goal is to maximizes overall AUC while keeping all the group-level AUC score to be the same as the overall AUC.

\begin{align*}\label{eq:equal-error}
\min_{\theta \in \Theta} & \hat{R}^\ell(\theta)\\
\text{s.t.} & \hat{R}_{z,z'}^\ell(\theta) = \hat{R}^\ell(\theta), z, z'=1, \cdots k. \numberthis
\end{align*}
Let $p_{z,z'} = \frac{n^{z+}n^{z'-}}{n^+n^-}$ be group proportions and it is straightforward to check that $\hat{R}^\ell(\theta) = \sum_{z=1}^{k}\sum_{z'=1}^{k} p_{z,z'} \hat{R}_{z,z'}^\ell(\theta)$. The Lagrangian dual function of Problem \eqref{eq:equal-error} is given by

\begin{align*}
F(\theta,\lambda) = & \lambda_{0,0} \hat{R}^\ell(\theta) + \sum_{z=1}^{k}\sum_{z'=1}^{k}\lambda_{z,z'} (\hat{R}_{z,z'}^\ell(\theta) - \hat{R}^\ell(\theta)) \\
= & (\lambda_{0,0} - \sum_{z=1}^{k}\sum_{z'=1}^{k}\lambda_{z,z'})\hat{R}^\ell(\theta) + \sum_{z=1}^{k}\sum_{z'=1}^{k}\lambda_{z,z'} \hat{R}_{z,z'}^\ell(\theta) \\
= & (\lambda_{0,0} - \sum_{z=1}^{k}\sum_{z'=1}^{k}\lambda_{z,z'})\sum_{z=1}^{k}\sum_{z'=1}^{k} p_{z,z'}\hat{R}_{z, z'}^\ell(\theta) + \sum_{z=1}^{k}\sum_{z'=1}^{k}\lambda_{z,z'} \hat{R}_{z,z'}^\ell(\theta) \\
= &  \sum_{z=1}^{k}\sum_{z'=1}^{k}\Big(\lambda_{z,z'} + \Big(\lambda_{0,0} - \sum_{z=1}^{k}\sum_{z'=1}^{k} \lambda_{z,z'}\Big) p_{z,z'}\Big) \hat{R}_{z,z'}^\ell(\theta) .
\end{align*}
When only intra-group AUCs are considered, we have $\lambda_{z, z'} = 0$ for all $z \neq z'$. When only inter-group AUCs are considered, we have $\lambda_{z, z} = 0$ for all $z = 1, \cdots k$. The empirical average of play converges to a Nash equilibrium, where an equilibrium corresponds to an optimal solution to the original constrained optimization problem \eqref{eq:equal-error}. It is worth noting that our equal error rates framework is different from the constrained optimization framework in \citet{narasimhan2020pairwise}. Firstly, we do not require a pre-defined "fairness level" $\kappa$ in the constraint of Eq. \eqref{eq:equal-error} whereas it is required in their work. Instead we require the group-level AUC to be the same as the overall AUC. This treatment can avoid fixing unrealistic small $\kappa$ and hurting the utility too much. Secondly, our training scheme is based on stochastic gradient descent ascent algorithm which is described in Algorithm \ref{alg:gda-fair-auc}, whereas \citeauthor{narasimhan2020pairwise} use a proxy loss function for updating the group weights. 

\begin{algorithm}[ht!]
\caption{\it EqualAUC\label{alg:gda-fair-auc}}
\begin{algorithmic}[1]
\STATE {\bf Inputs:} Training set $S$ with label $Y$ and protected attribute $Z$, model $f_\theta$, number of iterations $T$, batch size $m$, learning rates $\{\eta_\theta, \eta_\lambda\}$
\STATE {Initialize $\theta \in \Theta$ and $\lambda_{z, z'} = 0$}
\FOR{$t=1$ to $T-1$}
\STATE{$B_t = \texttt{Sample}_m(S)$}
\STATE{$\theta_t = \theta_{t-1} -\eta_\theta (\lambda_{t-1} + (1 - 1^\top\lambda_{t-1})p )^\top \partial_\theta\hat{R}^\ell(\theta_{t-1}; B_t)$}
\STATE{$\lambda_{z,z'} = \lambda_{z,z'} + \eta_\lambda(\hat{R}_{z,z'}^\ell(\theta) - \hat{R}^\ell(\theta))  $}
\ENDFOR
\STATE {\bf Outputs:} Uniform distribution over the set of models $\theta_1, \cdots, \theta_T$
\end{algorithmic}
\end{algorithm}

\section{Experimental Setup and Additional Experiments\label{sec:exp}}

In this section, we describe the empirical evaluation setup in details and provide more experimental results. 

\subsubsection{Implementation Details On Real Datasets.} We summarize implementation details of generating Table \ref{tab:comparison}.

\begin{itemize}
\item The \texttt{Adult}, \texttt{Bank} and \texttt{Compas} datasets are all pre-processed according to \citet{vogel2021learning}. The \texttt{Default} dataset is pre-processed according to \citet{donini2018empirical}. 
\item All experiments were implemented in Python with dependence on standard libraries: \texttt{pandas}, \texttt{NumPy}, \texttt{PyTorch}, \texttt{scikit-learn} and \texttt{matplotlib} for plots. All experiments were implemented on a server with the CPUs as Intel(R) Xeon(R) Silver 4214 @ 2.20GHz and the GPU as NVIDIA GeForce RTX 2080 Ti. 
\item All the hyperparameter are chosen based on 25 runs of the cross-validation with 60\% training, 20\% validation and 20\% testing. In particular, the batch size parameter $m$ is chosen from the candidate set $\{256, 512, 1024, 2048, 4096, 8192\}$ except $m=8192$ is too large to consider for $\texttt{Compas}$. The model stepsize parameter $\eta_\theta$ is chosen from the candidate set $\{0.2, 0.1, 0.05, 0.01\}$ and the stepsize for group weights, $\eta_\lambda$, is determined by $\eta_\lambda = \kappa \times \eta_\theta$ where the ratio $\kappa$ is chosen from the candidate set $\{100, 10, 1, 0.1, 0.01\}$. Furthermore, we also apply a weight decay parameter $\omega$ chosen from the candidate set $\{0.1, 0.01, 0.001\}$. 
\item All fairness-aware algorithms, i.e. \texttt{MinimaxFair}, \texttt{InterFairAUC}, \texttt{EqualAUC} and our \texttt{Algortihm \ref{alg:rmw-fair-auc}} initialize their model parameters $\theta_0$ via the same output of the pure AUC maximization algorithm \texttt{AUCMax} for fair comparison. All algorithms are trained on a simple neural network of 2 hidden layers. Each layer has the same width $d$ (the dimension of the input space), except for the output layer which outputs a real score. We used ReLU's as activation functions after each layer except for the output layer. To center and scale the output score we used batch normalization (BN) with fixed values $\gamma=1, \beta=0$ for the output value of the network. The surrogate loss function $\ell$ is chosen as the logistic loss, i.e. $\ell: s \mapsto \log(1 + \exp(-s))$.
\end{itemize}

\subsubsection{Generation Details For Synthetic Datasets.} Firstly, we discuss the generation of the scoring function in Figure \ref{fig:limitation-inter}. The distribution of the scoring function conditioned on the label $Y$ and protected attribute $Z$, i.e., $f_\theta|Y, Z$, is given as follow
\[
f_\theta |Y=y, Z=z \sim \text{logit}(\Ncal(\mu_{y,z}, \sigma^2_{y,z}))
\]
where logit is the logistic function. For every partition of a $(y,z)$ pair, we generate 1,000 scores. We apply $\mu_{0a} = 0.3, \mu_{1a} = 0.7, \mu_{0b} = 0.0, \mu_{0b} = 1.0$ and $\sigma_{y,z}^2 = 0.5$ for all $y \in \{\pm 1\}, z \in \{a, b\}$. Such choices allow the same gap between $\mu_{1a} - \mu_{0b}$ and $\mu_{1b} - \mu_{0a}$, leading towards fair inter-group AUCs. However, one has $\mu_{1a} - \mu_{0a} < \mu_{1b} - \mu_{0b}$, leading towards unfair intra-group AUCs in this one-dimensional example.

Next we discuss the generation of the two-dimensional dataset in Figure \ref{fig:synthetic}. The distribution of the feature vector $X$ conditioned on the label $Y$ and protected attribute $Z$ is given as follow
\[
X |Y=y, Z=z \sim \Ncal(\mu_{y,z}, \sigma^2_{y,z} \Ibb_{2\times 2})
\]
with $\mu_{0a} = (-1.0, 1.0)^\top, \mu_{1a} = (-1.5, 0.5)^\top, \mu_{0b} = (-2.0, -1.0)^\top, \mu_{1b} = (1.0, 0.0)^\top$ and $\sigma^2_{0a} = \sigma^2_{1a} = 0.5, \sigma^2_{0b} = \sigma^2_{1b} = 1.0$. For every partition of a $(y,z)$ pair, we generate 1,000 data points. Such distribution makes sure both intra-group and inter-group AUC unfairness exist while the intra-group AUC unfairness will be more severe.

\subsubsection{More On Ablation Studies.} We continue on reporting maximin AUC problem \eqref{eq:max-min-auc} with only intra-group AUCs or inter-group AUCs on \texttt{Bank}, \texttt{Compas} and \texttt{Default}. In general, we can roughly see only applying intra-group (resp. inter-group) AUC fairness may help but it is limited towards inter-group (resp. intra-group) AUC fairness.

\begin{figure*}[ht!]
\begin{subfigure}{0.49\linewidth}
    \includegraphics[width=\linewidth]{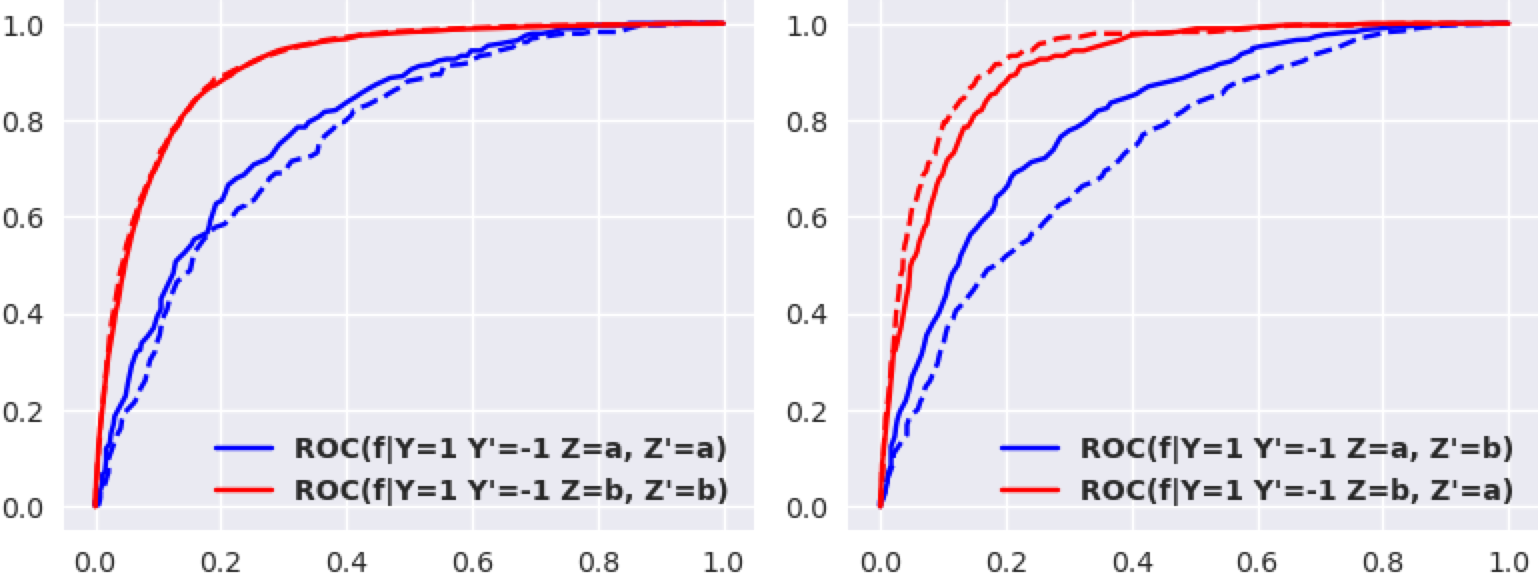}
    \caption{\it Intra-group AUCs only}
\end{subfigure}%
    \hfill%
\begin{subfigure}{0.49\linewidth}
    \includegraphics[width=\linewidth]{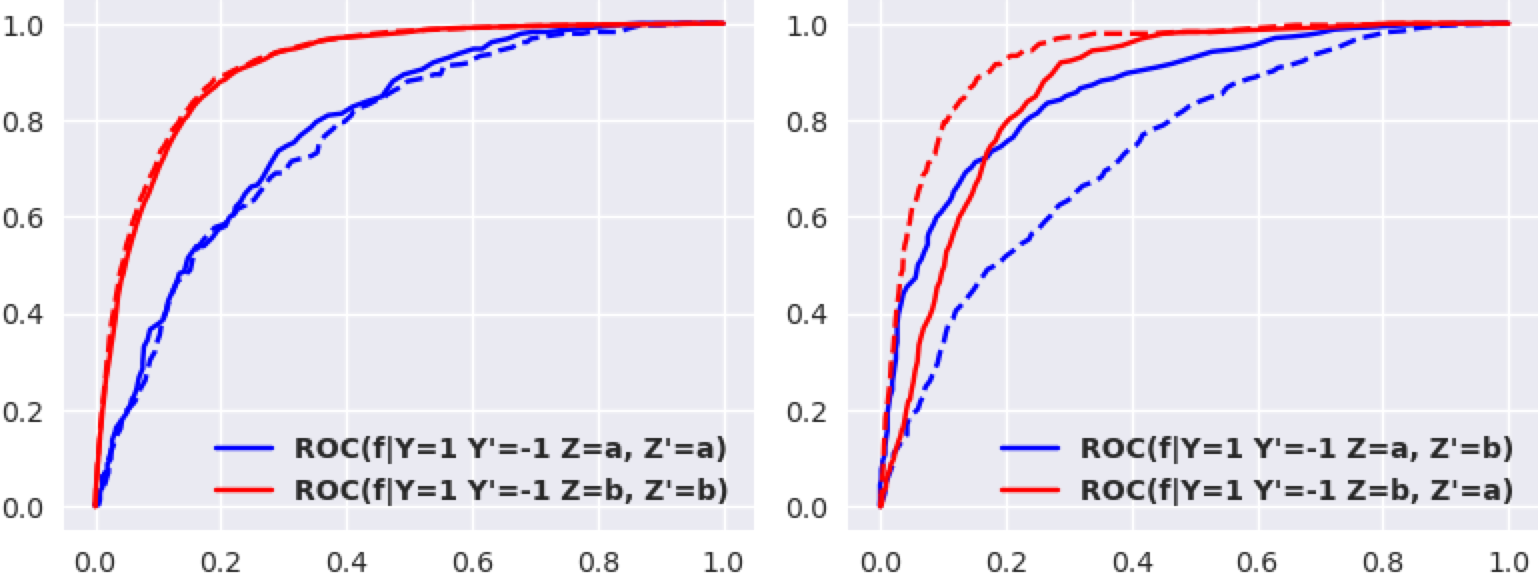}
    \caption{\it Inter-group AUCs only}
\end{subfigure}
\caption{\it Ablation study on \texttt{Bank}.}
\label{fig:bank-ablation}
\end{figure*}

\begin{figure*}[!ht]
\begin{subfigure}{0.49\linewidth}
    \includegraphics[width=\linewidth]{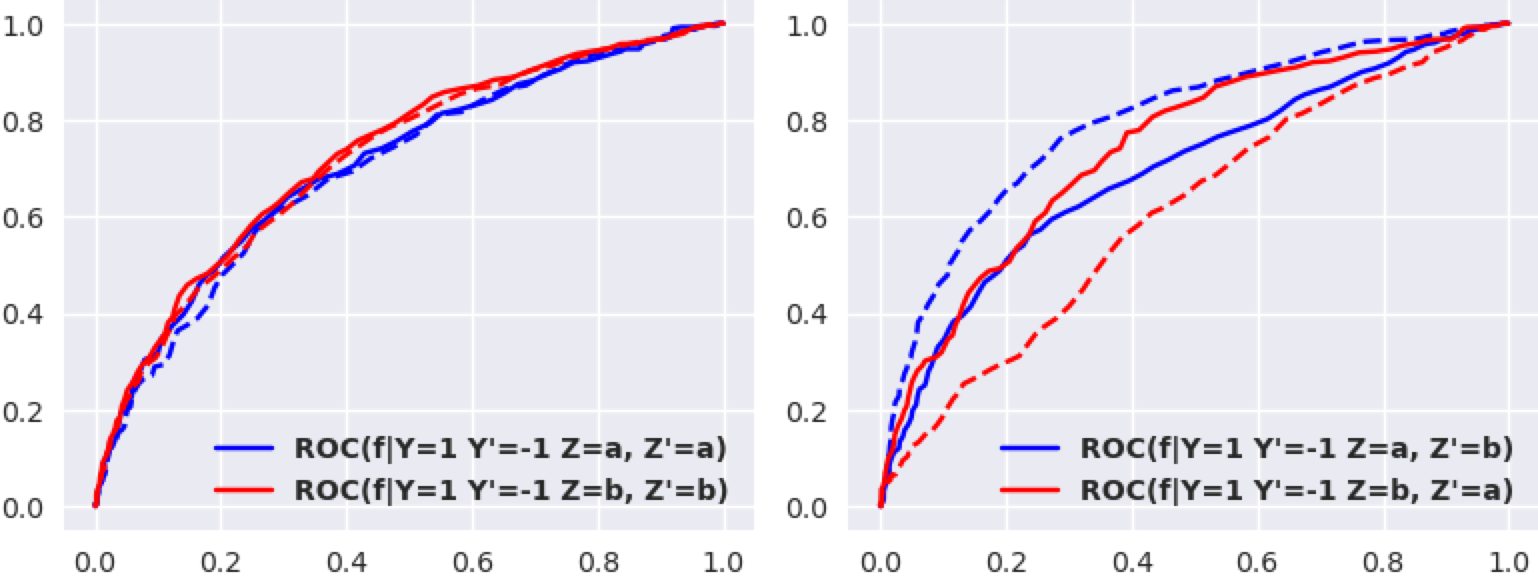}
    \caption{\it Intra-group AUCs only}
\end{subfigure}%
    \hfill%
\begin{subfigure}{0.49\linewidth}
    \includegraphics[width=\linewidth]{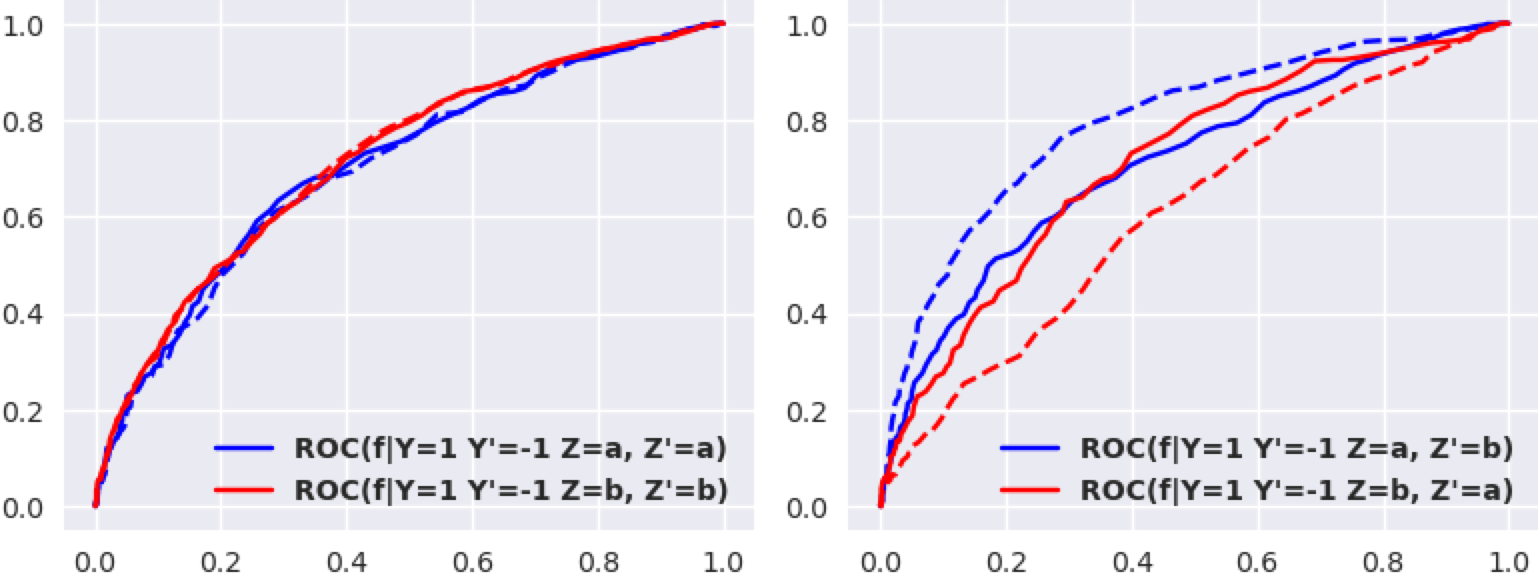}
    \caption{\it Inter-group AUCs only}
\end{subfigure}
\caption{\it Ablation study on \texttt{COMPAS}.}
\label{fig:compas-ablation}
\end{figure*}

\begin{figure*}[ht!]
\begin{subfigure}{0.49\linewidth}
    \includegraphics[width=\linewidth]{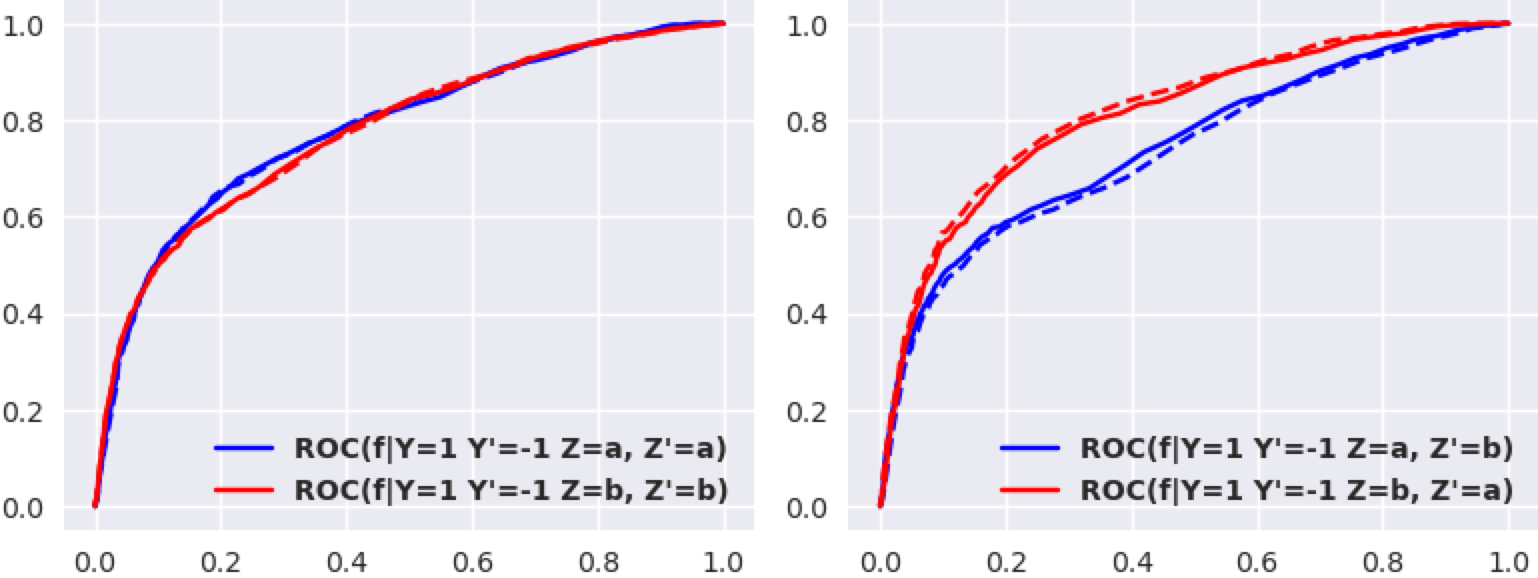}
    \caption{\it Intra-group AUCs only}
\end{subfigure}%
    \hfill%
\begin{subfigure}{0.49\linewidth}
    \includegraphics[width=\linewidth]{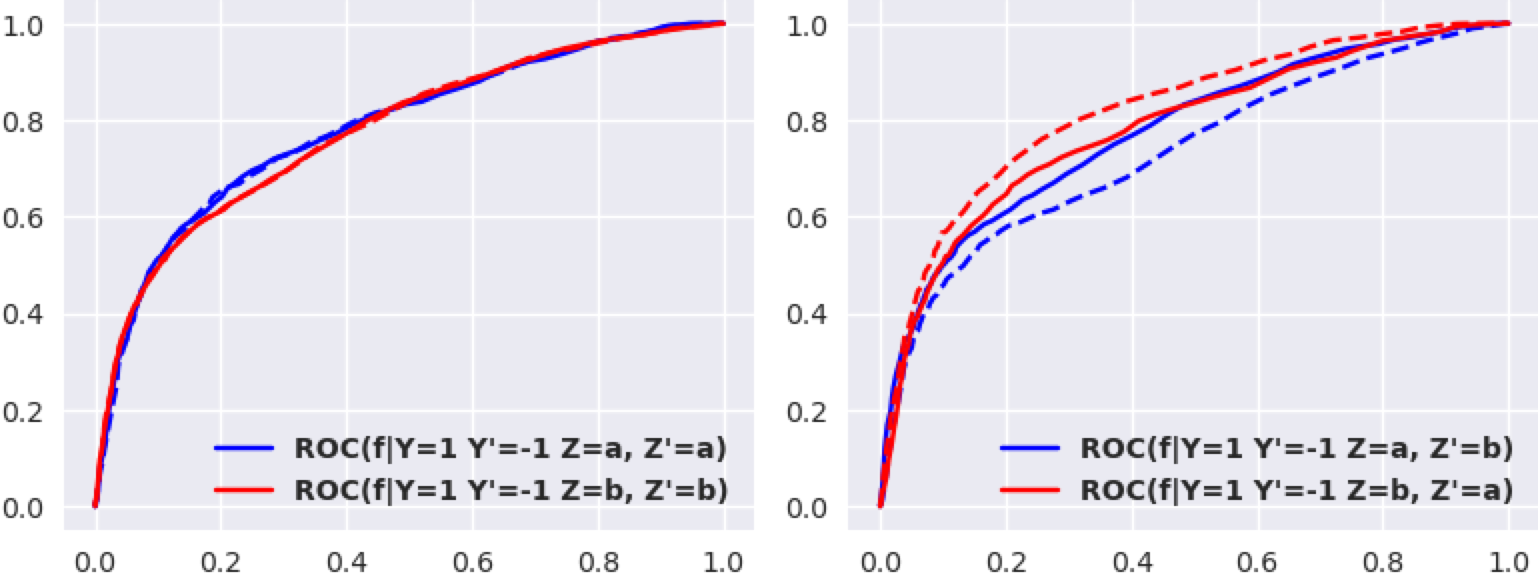}
    \caption{\it Inter-group AUCs only}
\end{subfigure}
\caption{\it Ablation study on \texttt{Default}.}
\label{fig:default-ablation}
\end{figure*}

\end{document}